\documentclass{article}

\usepackage{amsmath,amsfonts,bm}

\def\eqref#1{equation~\ref{#1}}

\def\1{\bm{1}}

\def\rvh{{\mathbf{h}}}

\def\rvr{{\mathbf{r}}}

\def\rvz{{\mathbf{z}}}

\def\rmA{{\mathbf{A}}}

\DeclareMathAlphabet{\mathsfit}{\encodingdefault}{\sfdefault}{m}{sl}
\SetMathAlphabet{\mathsfit}{bold}{\encodingdefault}{\sfdefault}{bx}{n}

\usepackage[preprint]{neurips_2025}

\PassOptionsToPackage{dvipsnames}{xcolor}
\usepackage{tcolorbox}
\usepackage{tabularray}
\usepackage{colortbl}
\usepackage{xcolor}
\usepackage[utf8]{inputenc} 
\usepackage[T1]{fontenc}    
\usepackage{hyperref}       
\usepackage{url}            
\usepackage{booktabs}       
\usepackage{amsfonts}       
\usepackage{nicefrac}       
\usepackage{microtype}      
\usepackage{tikz}
\usetikzlibrary{arrows}
\usepackage{amsmath}
\usepackage{amssymb}
\usepackage{mathtools}
\usepackage{amsthm}
\usepackage{algorithm,algpseudocode}
\usepackage{multirow}
\tcbuselibrary{listings}
\usepackage[utf8]{inputenc}
\usepackage{lipsum} 
\usepackage[normalem]{ulem}
\useunder{\uline}{\ul}{}
\usepackage{wrapfig}
\usepackage{graphicx}
\usepackage{caption}
\usepackage{subcaption}
\usepackage{cleveref}
\usepackage{pifont}
\usepackage{enumitem}
\usepackage{adjustbox}
\theoremstyle{plain}

\newtheorem{proposition}{Proposition}

\theoremstyle{definition}
\newtheorem{definition}{Definition}

\theoremstyle{remark}

\newcommand{\ti}[1]{\textit{#1}}

\newcommand{\tb}[1]{\textbf{#1}}

\newcommand{\mcal}[1]{\mathcal{#1}}

\newcommand{\patharrow}{\,\ensuremath{\to}\,} 

\newcommand{\proj}{\texttt{\scalebox{1.03}{R}APL}}

\title{Learning Efficient and Generalizable Graph Retriever for Knowledge-Graph Question Answering}

\author{%
  Tianjun Yao$^{1}$ \quad
  Haoxuan Li$^{1,2}$ \quad
  Zhiqiang Shen$^{1}$ \quad
  Pan Li$^{3}$ \quad
  \tb{Tongliang Liu}$^{4,1}$ \quad
  \tb{Kun Zhang}$^{1,5}$ \\[0.3em]
  $^{1}$Mohamed bin Zayed University of Artificial Intelligence \\
  $^{2}$Peking university \quad
  $^{3}$Georgia Institute of Technology \\
  $^{4}$The University of Sydney \quad
  $^{5}$Carnegie Mellon University \\
  \texttt{\{tianjun.yao,haoxuan.li,zhiqiang.shen\}@mbzuai.ac.ae}\\
  \texttt{pan.li@gatech.edu}, 
  \texttt{tongliang.liu@sydney.edu.au}, \texttt{kun.zhang@mbzuai.ac.ae} \\
}

\usepackage{etoc}
\etocdepthtag.toc{mtchapter}
\etocsettagdepth{mtchapter}{subsection}
\etocsettagdepth{mtappendix}{none}

\begin{document}

\maketitle

\begin{abstract}
  Large Language Models~(LLMs) have shown strong inductive reasoning ability across various domains, but their reliability is hindered by the outdated knowledge and hallucinations. Retrieval-Augmented Generation mitigates these issues by grounding LLMs with external knowledge; however, most existing RAG pipelines rely on unstructured text, limiting interpretability and structured reasoning. Knowledge graphs, which represent facts as relational triples, offer a more structured and compact alternative. Recent studies have explored integrating knowledge graphs with LLMs for knowledge graph question answering~(KGQA), with a significant proportion adopting the \ti{retrieve-then-reasoning} paradigm. In this framework, graph-based retrievers have demonstrated strong empirical performance, yet they still face challenges in generalization ability. In this work, we propose \proj, a novel framework for efficient and effective graph retrieval in KGQA. \proj{} addresses these limitations through three aspects: \ding{182} a two-stage labeling strategy that combines heuristic signals with parametric models to provide causally grounded supervision; \ding{183} a model-agnostic graph transformation approach to capture both intra- and inter-triple interactions, thereby enhancing representational capacity; and \ding{184} a path-based reasoning strategy that facilitates learning from the injected rational knowledge, and supports downstream reasoner through structured inputs. Empirically, \proj{} outperforms state-of-the-art methods by $2.66\%-20.34\%$, and significantly reduces the performance gap between smaller and more powerful LLM-based reasoners, as well as the gap under cross-dataset settings, highlighting its superior retrieval capability and generalizability. Codes are available at: \url{https://github.com/tianyao-aka/RAPL}.
\end{abstract}

\section{Introduction}
\label{sec:intro}
Large Language Models~(LLMs)~\cite{brown2020language,achiam2023gpt,touvron2023llama} have demonstrated remarkable capabilities in complex reasoning tasks across various domains~\cite{wu2024survey,fan2024hardmath,manning2024automated}, marking a significant step toward bridging the gap between human cognition and artificial general intelligence (AGI)~\cite{huang2023towards,wei2022chain,yao2024tree,bubeck2023sparks}. However, the reliability of LLMs remains a pressing concern due to outdated knowledge~\cite{kasai2023realtime} and hallucination~\cite{10.1145/3571730,huang2023a}. These issues severely undermine their trustworthiness in knowledge-intensive applications. 

To mitigate these deficiencies, Retrieval-Augmented Generation~(RAG)~\cite{gao2024retrievalaugmented,lewis2020retrieval} has been introduced to ground LLMs with external knowledge. While effective, most existing RAG pipelines rely on unstructured text corpora, which are often noisy, redundant, and semantically diffuse~\cite{shuster-etal-2021-retrieval-augmentation,gao2024retrievalaugmented}. In contrast, Knowledge Graph~(KG)~\cite{hogan2021knowledge} organizes information as structured triples $(h, r, t)$, providing a compact and semantically rich representation of real-world facts~\cite{chein2008graph,robinson2015graph}. As a result, incorporating KGs into RAG frameworks (i.e., KG-based RAG) has emerged as a vibrant and evolving area for achieving faithful and interpretable reasoning.

Building upon the KG-based RAG frameworks, recent studies have proposed methods that combine KGs with LLMs for Knowledge Graph Question Answering~(KGQA)~\cite{wang2023knowledge,dehghan-etal-2024-ewek,mavromatis2022rearev,luo2024rog,mavromatis2024gnn,chen2024plan,li2024simple}. A prevalent approach among these methods is the \ti{retrieve-then-reasoning} paradigm, where a retriever first extracts relevant knowledge from the KG, and subsequently an LLM-based reasoner generates answers based on the retrieved information. The retriever can be roughly categorized into LM-based retriever and graph-based retriever. Notably, recent studies have demonstrated that graph neural network~(GNN~\cite{kipf2016semi,hamilton2017inductive,velickovic2017graph,xu2018how})-based graph retrievers can achieve superior performance in KGQA tasks, even without fine-tuning the LLM reasoner~\cite{mavromatis2022rearev,li2024simple}. The success can be mainly attributed to the following factors: \ding{182} Unlike LLM-based retrievers, GNN-based retrievers perform inference directly on the KG, inherently mitigating hallucinations by grounding retrieval in the graph. \ding{183} GNNs are able to leverage the relational information within KGs, enabling the retrieval of contextually relevant triples that are crucial for accurate reasoning. Despite these success, we argue that current graph-based retrievers still face challenges that limit their effectiveness. Specifically, \ding{182} high-quality supervision is essential for machine learning models to generalize well. However, existing methods often rely on heuristic-based labeling strategies, such as identifying the shortest path between entities~\cite{luo2024rog,mavromatis2024gnn,li2024simple}. While seemingly reasonable, this approach can introduce noise and irrationalities. Specifically, \tb{(i)} multiple shortest paths may exist for a given question, not all of which are rational, and \tb{(ii)} some reasoning paths may not be the shortest ones. We provide 6 examples for each case in Appendix~~\ref{app:label-rationale} to support our claim. \ding{183} Although GNN-based retrievers inherently mitigate hallucinations, they may struggle to generalize to unseen questions, as they are not explicitly tailored to the unique characteristics of KGs and KGQA task, leading to \ti{limited generalization capacity}. Motivated by these challenges, we pose the following research question:

\ti{How to develop an efficient graph-based retriever that generalizes well for KGQA tasks?}

 To this end, we propose \proj, a novel framework that enhances the generalization ability of graph retrievers with \tb{R}ationalized \tb{A}nnotator, \tb{P}ath-based reasoning, and \tb{L}ine graph transformation.
 Specifically, \ding{182} instead of relying solely on heuristic-based labeling approaches such as shortest-path heuristics, we propose a two-stage labeling strategy. First, a heuristic-based method is employed to identify a candidate set of paths that are more likely to include rational reasoning paths. Then we obtain the causally grounded reasoning paths from this candidate set by leveraging the inductive reasoning ability of LLMs. \ding{183} \proj{} further improves the generalizability of the graph retriever via line graph transformation. This approach is model-agnostic and enriches triple-level representations by capturing both intra- and inter-triple interactions. Furthermore, it naturally supports path-based reasoning due to its directionality-preserving nature. \ding{184} We further introduce a path-based learning and inference strategy that enables the model to absorb the injected rational knowledge, thereby enhancing its generalizability. Additionally, the path-formatted outputs benefit the downstream reasoner by providing structured and organized inputs. In practice, our method outperforms previous state-of-the-art approaches by $2.66\%-20.34\%$ when paired with LLM reasoners of moderate parameter scale. Moreover, it narrows the performance gap between smaller LLMs (e.g., Llama3.1–8B) and more powerful models (e.g., GPT‑4o), as well as the gap under cross-dataset settings, highlighting its strong retrieval capability and generalization performance. 
 \vspace{-10pt}
\section{Preliminary}

\noindent\tb{triple $\tau$.}
A triple represents a factual statement:
$
\tau = \langle e, r, e' \rangle,
$
where $e, e' \in \mcal{E}$ denote the subject and object entities, respectively, and $r \in \mcal{R}$ represents the relation linking these entities.

\noindent\tb{Reasoning Path $p$.}
A reasoning path $p:=e_0 \xrightarrow{r_1} e_1 \xrightarrow{r_2} \cdots \xrightarrow{r_k} e_k$ connects a source entity to a target entity through one or more intermediate entities. Moreover, we denote $z_p:=\{r_1,r_2,\cdots r_k\}$ as the relation path of $p$.

\noindent\tb{Problem setup.} Given a natural language question $q$ and a knowledge graph $\mcal{G}$, our goal in this study is to learn a function $f_\theta$ that takes as inputs the question entity $e_q$, and a subgraph $\mcal{G}_q \subset \mcal{G}$, to infer an answer entity $e_a \in \mcal{G}_q$. Following previous practice, we assume that $e_q$ are correctly identified and linked in $\mcal{G}_q$.

\section{Related Work}
We discuss the relevant literature on the retrieval-then-reasoning paradigm and knowledge graph-based agentic RAG in detail in Appendix~\ref{related_work}.
 \begin{figure}[!t]  
  \centering
  \includegraphics[width=0.85\linewidth]{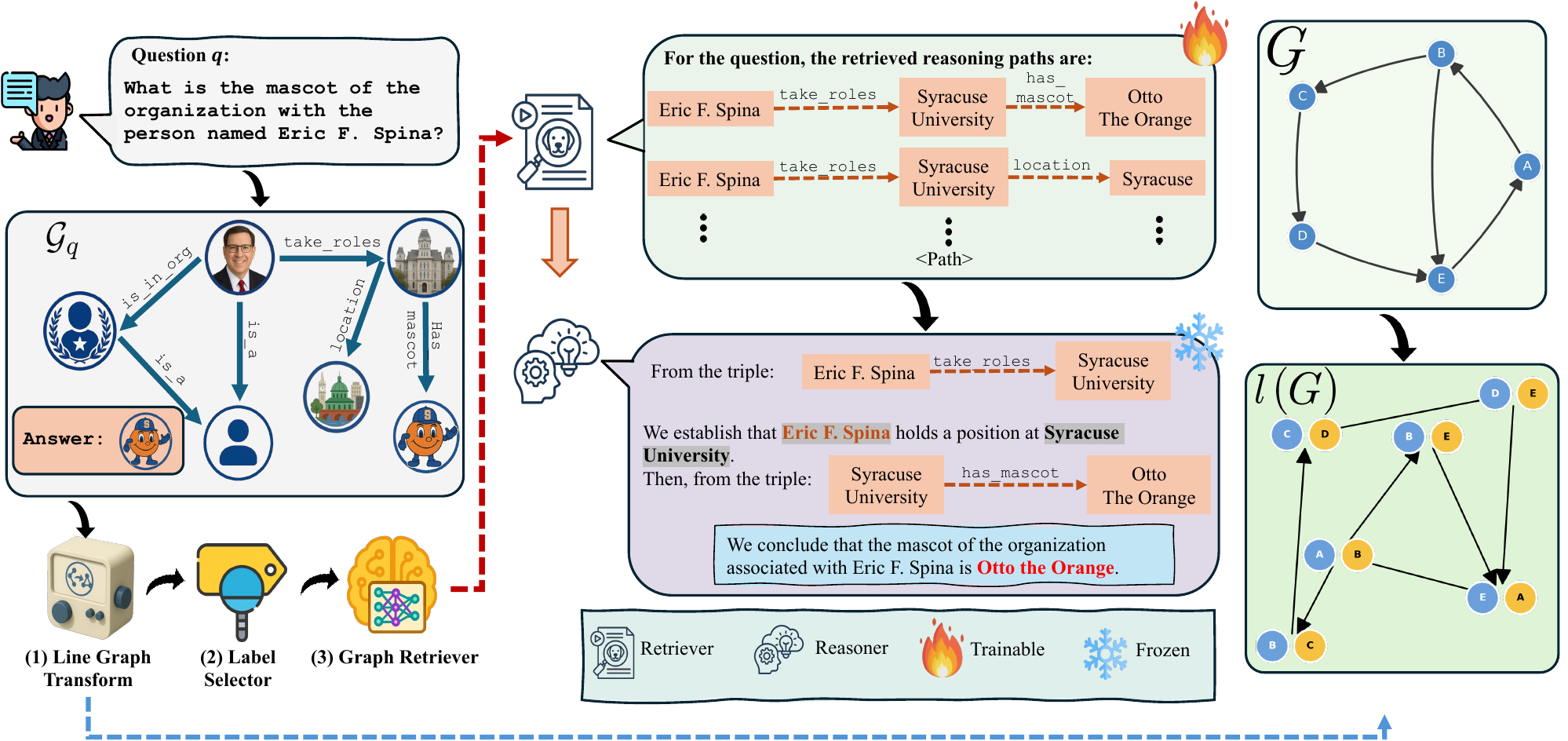}
  \caption{Overall framework of \proj. The generalization ability of \proj{} arises from the label rationalizer, line graph transformation, and the path-based reasoning paradigm.}
  \label{fig:model}
  \vspace{-8pt}
\end{figure}
\section{Empowering Graph Retrievers with Enhanced Generalization for KGQA} 
In this section, we discuss how to design an efficient and generalizable graph retriever from various aspects. The overall procedure is illustrated in Figure~\ref{fig:model}. Complexity analysis is deferred to Appendix~\ref{app:complexity}.

\subsection{Rationalized Label Supervision}

High-quality label supervision is crucial for a machine learning model to generalize. In KGQA, we define a labeling function $h(\cdot)$ such that $\widetilde{\mcal{Y}}_q = h(e_q, e_a, \mcal{G}_q)$, where $\widetilde{\mcal{Y}}_q$ represents the predicted paths or triples serving as the supervision signal. 

Previous studies often assume a heuristic-based labeling function $h$ that designates the shortest paths between $e_q$ and $e_a$ as the teacher signal. However, as discussed in Section~\ref{sec:intro}, a shortest path may not be a rational path, employing these noisy signals can undermine the generalization ability of the retrieval model. We show an example to demonstrate this point.

\tb{Example.} For the question: \texttt{What movie with film character named Woodson did Tupac star in?} Two candidate paths are: \ding{182}
$\textit{Tupac}  \xrightarrow[]{\texttt{film.actor}} \texttt{m.0jz0c4} \xrightarrow[]{\texttt{film.performance}} \textit{Gridlock’d}$, and  \ding{183}
$\textit{Tupac}  \xrightarrow[]{\texttt{music.recording}} \texttt{m.0jz0c4} \xrightarrow[]{\texttt{recording.release}} \textit{Gridlock’d}$.
The first path accurately captures the \ti{actor-character-film} reasoning chain, thereby offering a rational supervision signal. In contrast, the second path primarily reflects musical associations. Despite correctly linking the question entity to the answer entity, it fails to address the actual question intent. Learning from rational paths enhances the model’s ability to generalize to similar question types. More broadly, injecting causally relevant knowledge into the retrieval model contributes to improved generalization across diverse reasoning scenarios.

To inject the rational knowledge into the graph retriever, we incorporate a LLM-based annotator in the labeling function, reformulating it as:
\begin{equation}
\widetilde{\mcal{Y}}_q = h_{LM}(q, e_q, e_a, \mcal{G}_q, \gamma),
\end{equation}
where $\gamma$ denotes the LM-based reasoner. Specifically, our approach proceeds as follows. First, we generate a set of candidate paths $\mcal{P}_{cand}$, with the path length constrained to the range $[d_{min}, d_{min}+2]$, where $d_{min}$ is the shortest path distance from $e_q$ to $e_a$. Then, given the question $q$ and the candidate path set $\mcal{P}_{cand}$, we obtain the rational paths through $\widetilde{\mcal{Y}}_q = \gamma(q, \mcal{P}_{cand})$.
The labeling function $h_{LM}(\cdot)$ injects rational knowledge and provides more causally grounded supervision for training the graph retriever, thereby enhancing its generalization ability to unseen questions.

\subsection{Model-Agnostic Graph Transformation}
In this section, we propose to enhance the expressivity and generalizability of the graph retriever via line graph transformation.
We first introduce the definition of a \ti{directed} line graph.
\begin{definition} (\textit{Directed Line Graph}) \label{line_graph}
Given a directed graph $\mcal{G} = (\mcal{V}, \mcal{E})$, where each edge $e = (u, v) \in \mcal{E}$ has a direction from $u$ to $v$, the \textit{directed line graph} $l(\mcal{G})$ is a graph where:
\begin{itemize}
    \item Each node in $l(\mcal{G})$ corresponds to a directed edge in $\mcal{G}$.
    \item There is a directed edge from node $e_1 = (u, v)$ to node $e_2 = (v, w)$ in $l(\mcal{G})$ if and only if the target of $e_1$ matches the source of $e_2$ (i.e., $v$ is shared and direction is preserved).
\end{itemize}
\end{definition}
Transforming a directed graph $\mcal{G}_q$ into $\mcal{G}_q' = l(\mcal{G}_q)$ offers several advantages, as discussed below: 
\begin{itemize}[leftmargin=*]
    \item[\ding{182}] \textbf{Capturing Intra- and Inter-triple Interactions.}
    The line graph transformation enables more expressive modeling of higher-order interactions that are inherently limited in the original knowledge graphs when using message-passing GNNs. In the main paper, we adopt GCN~\cite{kipf2016semi} as a motivating example\footnote{Vanilla GCN doesn't use edge features in the update, here we adopt the most widely-used form for incorporating edge features.}, while further discussions on other widely-used GNN architectures are deferred to Appendix~\ref{discuss_graph_trasnform}. As shown in Eqn.~\ref{eq:mpnn}, GCN updates entity (node) representations by aggregating messages from neighboring entities and their associated relations. Although the node embedding $\rvh_i^{(k)}$ at layer $k$ incorporates relational information through an affine transformation applied to $\rvh_j^{(k-1)}$ and $\rvr_{ij}$, this formulation does not explicitly model the semantic composition of the full triple $(e_i, r_{ij}, e_j)$. Moreover, only the entity representations are iteratively updated, while the relation embedding remains static throughout propagation. This also constrains the model’s representational capacity, particularly in tasks that require fine-grained reasoning over relational structures.
    \begin{equation}
    \label{eq:mpnn}
    \rvh_i^{(k)} = \sigma \left( \sum_{j \in \mathcal{N}(i)} \tilde{\rmA}_{ij} \cdot \mathbf{W}^{(k-1)} \left[ \rvh_j^{(k-1)} \, \| \, \rvr_{ij} \right] \right),
    \end{equation}
    In contrast, the transformed line graph $\mcal{G}_q'$ treats each triple $(e_i, r_i, e_i')$ as a node, facilitating the iterative refinement of triple-level representations. As illustrated in Eqn.~\ref{eq:line_mpnn}, the message passing over $\mcal{G}_q'$ captures both: \tb{(i)} the intra-triple interactions among entities and the associated relations, and \tb{(ii)} the inter-triple dependencies by aggregating contextual information from neighboring triples in the line graph. Here, $\tilde{\rmA}'_{ij}$ denotes the normalized adjacency matrix in $\mcal{G}_q'$, $\mcal{N}_\ell(i)$ denotes the neighbors of node $i$ in the line graph, and $\phi(\cdot)$ is a aggregation function over triple components~(e.g., concatenation).
    \begin{equation}
    \label{eq:line_mpnn}
    \rvh_i^{(k)} = \sigma\left( \sum_{j \in \mathcal{N}_\ell(i)} \tilde{\rmA}_{ij}^\prime \cdot \phi\left(\mathbf{e}_j, \mathbf{r}_j, \mathbf{e}_j'\right) W_{agg}^{(k-1)} + \phi\left(\mathbf{e}_i, \mathbf{r}_i, \mathbf{e}_i'\right) W_{self}^{(k-1)} \right).
    \end{equation}
    \tb{Discussions.} Prior studies in various domains have highlighted the importance of relation paths for improving model generalization. For instance, in KGQA, RoG~\cite{luo2024rog} retrieves relation paths during the retrieval stage, while ToG~\cite{sunthink} relies on relation exploration to search on the KG. Beyond QA, recent advances in foundation models for KGs also leverage relational structures to enhance generalizability~\cite{galkin2023towards,geng2023relational,lee2023ingram,gao2023double,zhou2023ood}. However, how to effectively utilize relational structure within graph retrieval models remains largely underexplored. While previous studies~\cite{li2024simple,galkin2023towards,gao2023double} enhance model expressiveness through model-centric approaches such as labeling tricks~\cite{li2020distance,zhang2021labeling}, our approach instead leverages model-agnostic graph transformation to explicitly model inter-triple interactions, thereby addressing this gap.
    In particular, when only relation embeddings are learned from the triples, the model degenerates into learning relational structures. By introducing inter-triple message passing, \proj{} captures richer contextual triple-level dependencies that go beyond relation-level reasoning. Combined with causally grounded path-based supervision signals, the generalization capability of \proj{} is enhanced. 
    \item[\ding{183}] \textbf{Facilitating Path-based Reasoning.} The line graph perspective offers two key advantages for path-based learning and inference. 
\tb{(i)} The line graph transformation preserves \ti{edge directionality}, ensuring that any reasoning path $p$ in the original graph $\mcal{G}_q$ uniquely corresponds to a path $p_{\ell}$ in the transformed graph $\mcal{G}^{\prime}_q$. This property enables path-based learning and inference over $\mcal{G}^{\prime}_q$.
\tb{(ii)} In the original graph, path-based reasoning typically relies on entity embeddings. However, this approach suffers when encountering \ti{unknown entities} (e.g., m.48cu1), which are often represented as zero vectors or randomly initialized embeddings. Such representations lack semantic grounding and significantly impair the effectiveness of path-based inference. In contrast, the triple-level representations adopted in $\mcal{G}^{\prime}_q$ encodes both entity and relation semantics, thereby enriching the representational semantics.
\end{itemize}

\subsection{Path-based Reasoning}
In this section, we propose a path-based learning framework for \proj, which leverages the causally grounded rationales injected during the previous stage to enhance generalization. Moreover, the path-formatted inputs provide a more structured and logically coherent representation, thereby facilitating more effective reasoning by the downstream LLM-based reasoner (see Sec.~\ref{exp:in-depth}).
 
\tb{Training}. In the line graph perspective, each node $v_i$ represents a triple $(e_i, r_i, e_{i}^\prime)$. Consequently, a reasoning path for a given question $q$ is formulated as a sequence $(v_{q(0)}, v_{q(1)}, \ldots, v_{q(k)})$, where $v_{q(i)}$ denotes the vertex selected at step $i$.
Formally, the path reasoning objective is defined as:
\begin{equation}
\max_{\theta,\phi} \; \mathbb{P}_{\theta,\phi}\Bigl( v_{q(i)} \,\Big|\, v_{q(0)}, \ldots, v_{q(i-1)}, q \Bigr).
\end{equation}
Here, the node representation for each $v_i$ in $\mcal{G}_q'$ is obtained by $\mathbf{z}_i = f_\theta(v_i; \mcal{G}_q')$,
where $f_\theta(\cdot)$ is implemented as a GNN model. Furthermore, it is crucial to determine when to terminate reasoning. To address this, we introduce a dedicated \texttt{STOP} node, whose representation at step $i$ is computed as
\begin{equation}
\mathbf{z}^{\text{stop}}_{q(i)} = g_{\phi}( \text{AGG}\bigl(\mathbf{z}_{q(0)}, \mathbf{z}_{q(1)}, \ldots, \mathbf{z}_{q(i-1)}\bigr), \, q ),
\end{equation}
where $g_{\phi}(\cdot)$ is instantiated as a Multi-layer Perceptron (MLP), and $AGG$ denotes an aggregation function that integrates the node representations from the preceding steps, implemented as a summation in our work. The loss objective $\mcal{L}_{path}$ for every step $i$ can be shown as:

\begin{equation}
\label{path_loss}
\mcal{L}_{path} = \mathbb{E}_{\mcal{D}} \left[ -\log \frac{e^{\langle \mathbf{z}_q,\, \mathbf{z}_{q(i)}\rangle}}{\sum_{j \in \mcal{N}(q(i-1))} e^{\langle \mathbf{z}_q,\, \mathbf{z}_j\rangle}} \right], \text{ s.t. } i>0,
\end{equation}
where $\rvz_q$ is the representation for question $q$. 

\tb{Selecting question triple $v_{q(0)}$.} Equation~\ref{path_loss} applies only for $i>0$. At the initial step $i=0$, the model must learn to select the correct question triple $v_{q(0)}$ from all available candidate nodes. Since candidates are restricted to those involving the question entity $e_q$, the candidate space is defined as
\begin{equation}
\mcal{V}_{cand} := \{ v_i \mid v_i = (e_i, r_i, e_i'), \, e_i = e_q \}.
\end{equation}
Although one can apply a softmax loss over $\mcal{V}_{cand}$, in practice this set often contains hundreds or even thousands of nodes, while the ground-truth $v_{q(0)}$ is typically unique or limited. This imbalance can lead to ineffective optimization and inference.

To mitigate this issue, we introduce \textit{positive sample augmentation}. Specifically, given the question representation $\rvz_q$ and the set of relations present in $\mcal{G}_q$, a reasoning model (e.g., GPT-4o-mini) is employed to select the most probable relation set for the question $q$, denoted as $\mcal{R}_*$. The augmented positive vertex set is then defined as:
\begin{equation}
\mcal{V}_{pos} := \{ v_{q(0)} \} \cup \{ v_i \mid v_i=\langle e_i,r_i,e_i^\prime \rangle,e_i = e_q, \, r_i \in \mcal{R}_* \},
\end{equation}
and the negative vertex set is given by $\mcal{V}_{neg} := \mcal{V}_{cand} \setminus \mcal{V}_{pos}$.
We adopt a negative sampling objective~\cite{mikolov2013efficient} for optimizing the selection of the initial triple. The question triple selection loss $\mcal{L}_q$ is formulated as: 
\begin{equation}
\label{eq:L_q_neg_sampling}
\mathcal{L}_q =
\mathbb{E}_{q \sim \mathcal{D}}
\Biggl[
-\,
\frac{1}{\lvert \mathcal{V}_{pos} \rvert}
\sum_{v^{+}\in\mathcal{V}_{pos}}
\log \sigma\!\bigl( \langle \mathbf{z}_{q}, \mathbf{z}_{v^{+}} \rangle \bigr)
\;-\;
\frac{1}{\lvert \mathcal{V}_{neg} \rvert}
\sum_{v^{-}\in\mathcal{V}_{neg}}
\log \sigma\!\bigl( -\langle \mathbf{z}_{q}, \mathbf{z}_{v^{-}} \rangle \bigr)
\Biggr],
\end{equation}
where $\sigma(x) = 1 / (1 + e^{-x})$ is the logistic sigmoid function. 

\tb{Look-ahead embeddings.} As $\mcal{G}_q^\prime$ is an directed graph, the node representation $\rvz_i, i \in \mcal{V}$ can only incorporate information from its predecessors $v_{q(0)}, \dots, v_{q(i-1)}$. This can be suboptimal, especially for earlier nodes along the reasoning path, since they cannot utilize the information from subsequent nodes to determine the next action. To address this issue, we introduce a \emph{look-ahead} message-passing mechanism by maintaining two sets of model parameters $\theta:= \{\overset{\rightarrow}{\theta},\overset{\leftarrow}{\theta}\}$, which acts on $\mcal{G}_q^\prime$ and its edge-reversed counterpart $\overset{\leftarrow}{\mcal{G}^\prime}_q$ respectively.
\begin{equation}
\label{eq:look_ahead}
\begin{aligned}
\overset{\rightarrow}{\mathbf{z}}_{i} 
&= f_{\overset{\rightarrow}{\theta}}\bigl(v_i;\,\mcal{G}_q'\bigr),
\\[1pt]
\overset{\leftarrow}{\mathbf{z}}_{i} 
&= f_{\overset{\leftarrow}{\theta}}\bigl(v_i;\,\overset{\leftarrow}{\mcal{G}^\prime}_q\bigr),
\\[1pt]
\mathbf{z}_i 
&= MEAN(\overset{\rightarrow}{\mathbf{z}}_{i}, \overset{\leftarrow}{\mathbf{z}}_{i}).
\end{aligned}
\end{equation}

\tb{Inference.} During the inference stage, we first sample the the question triple $\widetilde{v}_{q(0)}$ for $K$ times with replacement. For each of these $K$ candidates, we then sample 5 reasoning paths, we then choose $M$ retrieved paths with highest probability score, followed by deduplication to eliminate repeated paths. The resulting set of unique reasoning paths is passed to the reasoning model to facilitate knowledge-grounded question answering.

\section{Experiments}
\label{sec:exp}
\begin{wraptable}{r}{0.55\textwidth}  
\vspace{-0.3cm}
\caption{
Test performance on WebQSP and CWQ. The best results are highlighted in \textbf{bold}, and the second-best in \underline{underline}. We use \textcolor{red}{red}, \textcolor{blue}{blue}, and \textcolor{green}{green} shading to indicate the best-performing result within each retrieval configuration. $(X, Y)$ denotes the average number of retrieved triples on WebQSP and CWQ respectively.}
\label{tab:main_result}
\centering
\begin{adjustbox}{width=\linewidth}
\begin{tabular}{clcccc}  
\toprule
\multirow{2}{*}{} & \multirow{2}{*}{Method} & \multicolumn{2}{c}{WebQSP} & \multicolumn{2}{c}{CWQ} \\
\cmidrule(lr){3-4}\cmidrule(lr){5-6}
 & & Macro-F1 & Hit & Macro-F1 & Hit \\
\midrule
\multirow{5}{*}{\rotatebox[origin=c]{90}{LLM}} 
  & Qwen-7B~\cite{qwen2} &   35.5 & 50.8     & 21.6 &  25.3     \\
  & Llama3.1-8B~\cite{llama3} &   34.8 & 55.5   & 22.4 &  28.1     \\
  & GPT-4o-mini~\cite{gpt4o} &   40.5 & 63.8     & 40.5 &  63.8     \\
  & ChatGPT~\cite{chatgpt} &   43.5 & 59.3     & 30.2 &  34.7     \\
  & ChatGPT+CoT~\cite{wei2022chain} &   38.5 & 73.5     & 31.0 &  47.5     \\
\midrule
\multirow{10}{*}{\rotatebox[origin=c]{90}{KG+LLM}} 
  & UniKGQA~\cite{jiang2022unikgqa} &   72.2 & --     & 49.0 & --     \\
  & KD-CoT~\cite{wang2023knowledge} &   52.5 & 68.6 & --     & 55.7 \\
  & ToG (GPT-4)~\cite{sunthink} &   --     & 82.6 & --     & 67.6 \\
  & StructGPT~\cite{jiang2023structgpt} &     --     & 74.6 & --     & --     \\
  & Retrieve-Rewrite-Answer~\cite{wu2023retrieverewriteanswerkgtotextenhancedllms} & --     & 79.3 & --     & --     \\
  & G-Retriever~\cite{he2024g} &   53.4 & 73.4 & --     & --     \\
  & RoG~\cite{luo2024rog} &    70.2 & 86.6 & 54.6 & 61.9 \\
  & EtD~\cite{liu2024explore} &    --     & 82.5 & --     & 62.0 \\
  & GNN-RAG~\cite{mavromatis2024gnn} &   71.3 & 85.7 & \tb{59.4} & 66.8 \\
  & SubgraphRAG + Llama3.1-8B~\cite{li2024simple} &  70.5 & 86.6 & 47.2 & 56.9 \\
  & SubgraphRAG + GPT-4o-mini~\cite{li2024simple} &  77.4 & 90.1 & 54.1 & 62.0 \\
  & SubgraphRAG + GPT-4o~\cite{li2024simple} &        76.4 & 89.8 & \underline{59.1} & 66.6 \\
\midrule
  & Ours + Llama3.1-8B~(24.87, 28.53) & 74.8 & 87.8 & 48.6 & 57.6 \\
  & Ours + GPT-4o-mini~(24.87, 28.53) & \cellcolor{red!15}79.8 & 92.0 & 56.6 & \cellcolor{red!15}68.0 \\
  & Ours + GPT-4o~(24.87, 28.53) & 79.4 & \cellcolor{red!15}\underline{93.0} & \cellcolor{red!15}56.7 & \cellcolor{red!15}68.0 \\ \midrule
  & Ours + Llama3.1-8B~(31.89, 38.76) & 76.2 & 88.3 & 56.1 & 66.7 \\
  & Ours + GPT-4o-mini~(31.89, 38.76) & \cellcolor{blue!10}79.2 & 92.2 & 57.2 & \cellcolor{blue!10}68.9 \\
  & Ours + GPT-4o~(31.89, 38.76) & \cellcolor{blue!10}79.2 & \cellcolor{blue!10}92.3 & \cellcolor{blue!10}58.3 & 68.8 \\ \midrule
  & Ours + Llama3.1-8B~(41.51, 52.45) & 77.3 & 88.8 & 56.8 & 67.2 \\
  & Ours + GPT-4o-mini~(41.51, 52.45) & \underline{80.4} & 92.5 & 58.1 & \cellcolor{green!15}\tb{69.3} \\
  & Ours + GPT-4o~(41.51, 52.45) & \cellcolor{green!15}\tb{80.7} & \cellcolor{green!15}\tb{93.3} & \cellcolor{green!15}58.8 & \underline{69.0} \\
\bottomrule
\end{tabular}
\end{adjustbox}
\vspace{-10pt}
\end{wraptable}
This section evaluates our proposed method by answering the following research questions: \textbf{RQ1:} How does our method compare to state-of-the-art baselines in overall performance?
\textbf{RQ2:} How do various design choices in \proj{} influence performance?
\textbf{RQ3:} How does our method perform on questions with different numbers of reasoning hops?
\textbf{RQ4:} Can the path-structured inputs enhance the performance of downstream LLM-based reasoning modules?
\textbf{RQ5:} How faithful and generalizable is our method?
Additionally, we provide efficiency analysis and demonstrations of retrieved reasoning paths in Appendix~\ref{app:exp}.
\subsection{Experiment Setup}
\noindent\tb{Datasets.} We conduct experiments on two widely-used and challenging benchmarks for KGQA: WebQSP~\citep{yih2016value} and CWQ~\citep{talmor2018web}, Both datasets are constructed to test multi-hop reasoning capabilities, with questions requiring up to four hops of inference over a large-scale knowledge graph. The underlying knowledge base for both is Freebase~\citep{bollacker2008freebase}. Detailed dataset statistics are provided in Appendix~\ref{dataset}. 

\noindent\tb{Baselines.} We compare \proj{} with 15 state-of-the-art baseline methods, encompassing both general LLM without external KGs, and KG-based RAG approaches that integrate KGs with LLM for KGQA. Among them, GNN-RAG and SubgraphRAG utilize graph-based retrievers to extract relevant knowledge from the knowledge graph. For SubgraphRAG, we adopt the same reasoning modules as used in our framework to ensure fair comparisons. Since \proj{} retrieves fewer than 100 triples in the experiments, we report the performance of SubgraphRAG using the setting with 100 retrieved triples, as provided in its original paper~\cite{li2024simple}.

\noindent\tb{Evaluation.} Following previous practice, we adopt Hits and Macro-F1 to assess the effectiveness of \proj. \textit{Hits} measures whether the correct answer appears among the predictions, while \textit{Macro-F1} computes the average of F1 scores across all test samples, providing a balanced measure of precision and recall.

\noindent\tb{Setup.} Following prior work, we employ \ti{gte-large-en-v1.5}~\cite{li2023towards} as the pretrained text encoder to extract text embeddings to ensure fair comparisons. For reasoning path annotation, we use GPT-4o, and conduct ablation studies on alternative annotators in Sec.~\ref{sec:label_annotator}. The graph retriever adopted is a standard GCN without structural modifications. We evaluate three configurations for the retrieval parameters $(K, M)$, which result in varying sizes of retrieved triple sets: $\{K=60, M=80\}$, $\{K=80, M=120\}$, and $\{K=120, M=200\}$. For the reasoning module, we consider GPT-4o, GPT-4o-mini, and instruction-tuned Llama3.1-8B model without fine-tuning efforts.
\subsection{Main Results}
From Table~\ref{tab:main_result}, we can make several key observations:
\ding{182} General LLM methods consistently underperform compared to KG-based RAG approaches, as they rely solely on internal parametric knowledge for reasoning.
\ding{183} Among KG-based RAG methods, \proj{} achieves superior performance across most settings, significantly outperforming competitive baselines, particularly when paired with moderately capable LLM reasoners such as GPT-4o-mini and LLaMa3.1-8B.
\ding{184} Compared to SubgraphRAG, \proj{} exhibits a notably smaller performance gap between strong and weak reasoners. For instance, on the CWQ dataset, the Macro-F1 gap between GPT-4o and LLaMA3.1-8B is reduced from $14.78\%$ (for SubgraphRAG) to $2.22\%$ with \proj{}, indicating that our retriever effectively retrieves critical information in a more compact and efficient manner. 
\begin{wraptable}{r}{0.5\textwidth}
\caption{
 The impact of different label annotation methods on WebQSP and CWQ. $K=60,M=80$ is used in the experiments.
}
\label{tab:teacher_supervision}
\centering
\begin{adjustbox}{width=\linewidth}
\begin{tabular}{llccccc}
\toprule
\multirow{2}{*}{Label Annotator} &  & \multicolumn{2}{c}{WebQSP} & \multicolumn{2}{c}{CWQ} \\
\cmidrule(lr){3-4} \cmidrule(lr){5-6}
 & & Macro-F1 & Hit & Macro-F1 & Hit \\
\midrule
\multirow{2}{*}{GPT-4o} 
  & GPT-4o-mini     & \textbf{79.8}  & \textbf{92.0}  & \textbf{56.8} & \textbf{68.0} \\
  & Llama3.1-8B    & \underline{74.8} & 87.8  & 48.6 & 57.7 \\
\midrule
\multirow{2}{*}{GPT-4o-mini} 
  & GPT-4o-mini     & \underline{78.6} & 91.6  & 54.6 & \underline{64.2} \\
  & Llama3.1-8B    & 73.7 & 87.8  & 46.0 & 55.4 \\
\midrule
\multirow{2}{*}{ShortestPath} 
  & GPT-4o-mini     & \underline{78.6} & \underline{91.7}  & \underline{54.7} & 63.8 \\
  & Llama3.1-8B    & 74.3 & 88.7  & 46.3 & 55.2 \\
\bottomrule
\end{tabular}
\end{adjustbox}
\end{wraptable}
\ding{185} When pairing with Llama3.1-8B and GPT-4o-mini, \proj{} outperforms SubgraphRAG by $9.61\%\! \uparrow$ and $3.91\% \!\uparrow$ in terms of Macro-F1 in CWQ dataset, with $50\% \!\!\downarrow$ retrieved triples. 
\ding{186} For GPT-4o, a high-capacity reasoner, \proj{} slightly underperforms SubgraphRAG and GNN-RAG. We hypothesize that this is due to GPT-4o’s strong denoising ability, which allows it to benefit from longer, noisier inputs. Moreover, although shortest-path-based labeling may introduce noise, it ensures broader coverage of information, which may benefit powerful reasoners. To test this hypothesis, we perform inference with an extended budget up to 100 triples trained using shortest path labels as well as rational labels, resulting in a Macro-F1 of $59.61\%$ on CWQ. This suggests that labeling strategies should be adapted based on the reasoning capacity of the downstream LLM.

\subsection{Ablation Study}
\label{sec:label_annotator}
In this section, we investigate the impact of different label annotators and different architectural designs in the graph retriever. We use $K=60,M=80$ during inference time to conduct the experiments on both datasets.

\tb{Impact of label annotators.} While our main results are obtained using rational paths labeled by GPT-4o, we further evaluate two alternative strategies: GPT-4o-mini and shortest path.
As shown in Table~\ref{tab:teacher_supervision}, Using GPT-4o as the label annotator yields the best overall performance across both datasets. Its advantage is particularly evident on the more challenging CWQ dataset, which requires multi-hop reasoning. Specifically, with GPT-4o-mini as the downstream reasoning module, GPT-4o-labeled data achieves absolute gains of $2.16\%$ and $2.07\%$ in Macro-F1, and $3.74\%$ and $4.15\%$ in Hit, compared to using GPT-4o-mini and shortest-path heuristics respectively. A similar trend is observed when Llama3.1-8B is used as the reasoner.
In contrast, GPT-4o-mini as a label annotator consistently underperforms GPT-4o and yields results comparable to those derived from shortest-path supervision. These findings indicate that a language model with limited reasoning capacity may struggle to extract causally grounded paths, leading to suboptimal downstream performance.
\begin{wraptable}{r}{0.48\textwidth}  
    \vspace{-0.1cm}
    \caption{
        Performance of different graph retrievers. Best results are in \textbf{bold}.
    }
    \label{tab:graph_retriever}
    \centering
    \begin{adjustbox}{width=\linewidth}
        \begin{tabular}{llcccc}
            \toprule
            \multirow{2}{*}{Graph Retriever} &  & \multicolumn{2}{c}{WebQSP} & \multicolumn{2}{c}{CWQ} \\
            \cmidrule(lr){3-4} \cmidrule(lr){5-6}
             & & Macro-F1 & Hit & Macro-F1 & Hit \\
            \midrule
            \multirow{2}{*}{1-layer GCN} 
                 & GPT-4o-mini  & 74.0 & \underline{90.4} & 51.4 & 62.5 \\
                 & Llama3.1-8B  & 67.7 & 84.1 & 44.2 & 53.1 \\
            \midrule
            \multirow{2}{*}{2-layer GCN (w/o look-ahead)} 
                 & GPT-4o-mini  & 74.4 & 89.1 & \underline{51.9} & 61.0 \\
                 & Llama3.1-8B  & 69.1 & 84.6 & 43.2 & 51.6 \\
            \midrule
            \multirow{2}{*}{2-layer GCN} 
                 & GPT-4o-mini  & \textbf{79.8} & \textbf{92.0} & \textbf{56.8} & \textbf{68.0} \\
                 & Llama3.1-8B  & \underline{74.8} & 87.9 & 48.6 & \underline{57.7} \\
            \bottomrule
        \end{tabular}
    \end{adjustbox}
\end{wraptable}
\tb{Impact of graph retriever architectures.}
We analyze how different GNN architectures used in the graph retriever affect KGQA performance. In the experiment, the label annotator is fixed to GPT-4o, and we evaluate performance using GPT-4o-mini and Llama3.1-8B as the reasoning models. As shown in Table~\ref{tab:graph_retriever}, a 1-layer GCN performs similarly as a 2-layer GCN 
without the look-ahead mechanism. However, when the look-ahead message passing is incorporated into the 2-layer GCN, KGQA performance improves significantly across both datasets and reasoning models. This highlights the importance of the look-ahead design, which facilitates both reasoning paths selection and question triples selection by incorporating information from subsequent nodes along the reasoning paths. 
\vspace{-10pt}
\subsection{In-Depth Analysis}
\label{exp:in-depth}
\tb{Performance analysis on varying reasoning hops.} We evaluate the performance of \proj{} on test samples with varying numbers of reasoning hops under the three experimental settings introduced earlier. We adopt Llama3.1-8B and GPT-4o-mini as the downstream reasoners, and restrict the evaluation to samples whose answer entities are present in the KG. From Table~\ref{tab:acc_hop}, we can make several key observations:
\ding{182} On the WebQSP dataset, \proj{} achieves strong performance even with only 24.87 retrieved triples, particularly for samples requiring two-hop reasoning. When increasing the retrieval to 41.51 triples, \proj{} outperforms the second-best method in terms of Macro-F1 by $12.39\%$ with Llama3.1-8B and $4.44\%$ with GPT-4o-mini.
\ding{183} On the CWQ dataset, when the number of retrieved triples exceeds 38 ($K=80,M=120$), \proj{} consistently surpasses previous methods across all hops when using Llama3.1-8B as the reasoner. Notably, for queries requiring $\geq 3$ hops, \proj{} achieves substantial gains over SubgraphRAG, improving Macro-F1 by $40.28\%$ and $18.97\%$ with Llama3.1-8B and GPT-4o-mini respectively.
\begin{wrapfigure}{r}{0.4\textwidth}
\vspace{-0.5em}
\centering
\includegraphics[width=\linewidth]{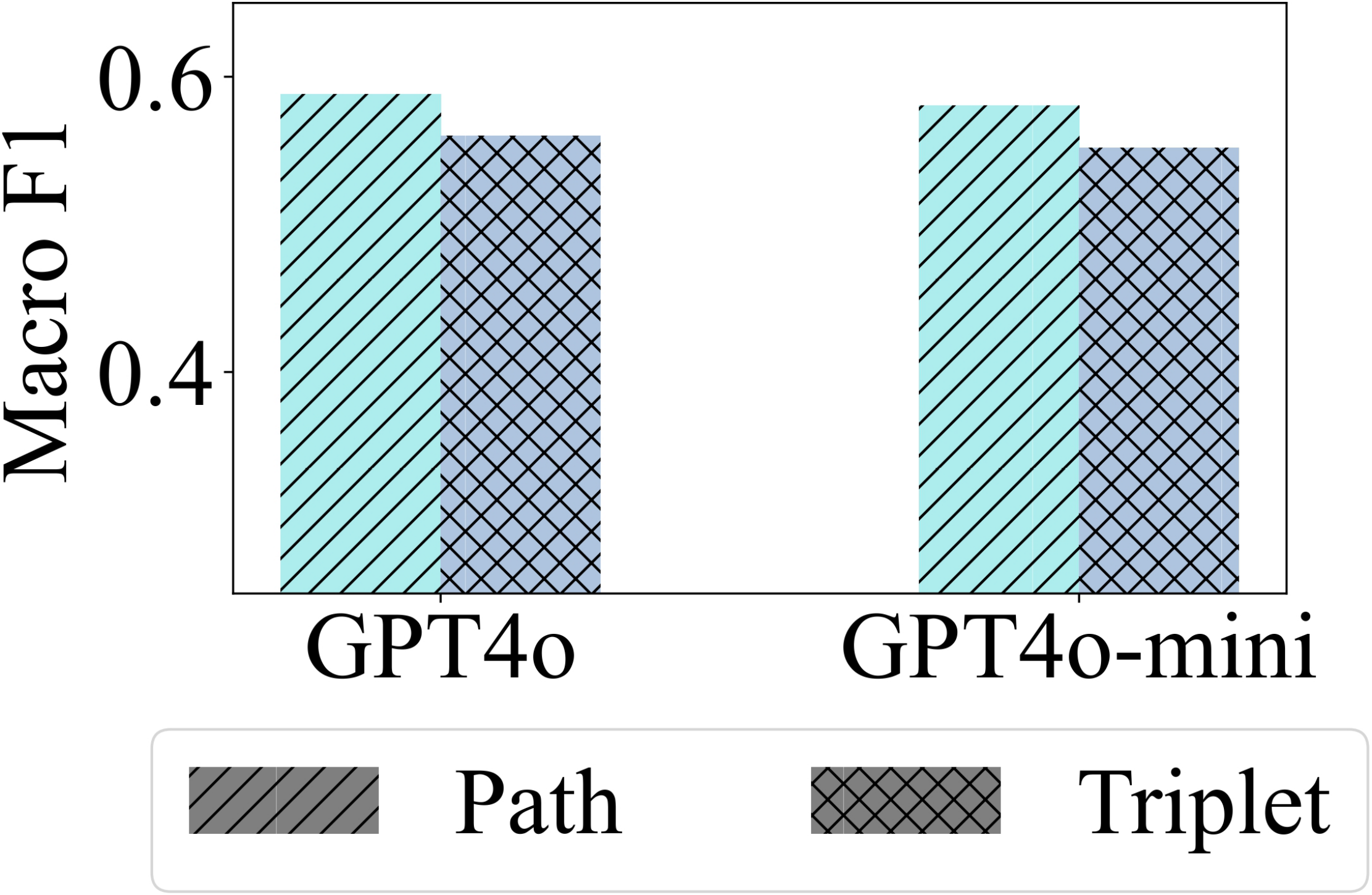}
\vspace{-1.0em}
\caption{
Impact of path-formatted inputs on reasoning performance.
}
\label{fig:path_format_effect}
\vspace{-1.5em}
\end{wrapfigure}
\ding{184} Finally, for 2-hop samples in the CWQ dataset, we observe that when using Llama3.1-8B as the reasoner, \proj{} outperforms SubgraphRAG. However, when paired with GPT-4o-mini, \proj{} underperforms in terms of Macro-F1, despite retrieving the same set of triples. This suggests that GPT-4o-mini possesses stronger denoising capabilities and is better able to reorganize unstructured triples into coherent reasoning paths. Moreover, the relatively small difference in Hit but large gap in Macro-F1 indicates that the weakly supervised labeling strategy used by SubgraphRAG may expose the reasoner to more diverse knowledge. As a result, the increased number of retrieved triples may better facilitate GPT-4o-mini than Llama3.1-8B.
\begin{table}[!t]
    \caption{Breakdown of QA performance by reasoning hops.}
    \label{tab:acc_hop}
    \centering
    \begin{adjustbox}{width=0.9\textwidth}
        \begin{tabular}{ccccccccccc}
            \toprule
             & \multicolumn{4}{c}{WebQSP} & \multicolumn{6}{c}{CWQ}\\
            \cmidrule(lr){2-5} \cmidrule(lr){6-11}
              & \multicolumn{2}{c}{$1$}  
              & \multicolumn{2}{c}{$2$}  
              & \multicolumn{2}{c}{$1$}  
              & \multicolumn{2}{c}{$2$} 
              & \multicolumn{2}{c}{$\geq 3$}  \\
              & \multicolumn{2}{c}{$(65.8 \%)$}  
              & \multicolumn{2}{c}{$(34.2 \%)$} 
              & \multicolumn{2}{c}{$(28.0 \%)$}  
              & \multicolumn{2}{c}{$(65.9 \%)$}  
              & \multicolumn{2}{c}{$(6.1 \%)$}  \\
            \cmidrule(lr){2-3} \cmidrule(lr){4-5} \cmidrule(lr){6-7} \cmidrule(lr){8-9} \cmidrule(lr){10-11} 
             & Macro-F1 & Hit & Macro-F1 & Hit & Macro-F1 & Hit & Macro-F1 & Hit & Macro-F1 & Hit \\
            \midrule
            G-Retriever 
            & 56.4 & 78.2 & 45.7 & 65.4 & - & - & - & - & - & - \\
            RoG
            & 77.1 & 93.0 & 62.5 & 81.5 & 59.8 & 66.3 & 59.7 & 68.6 & 41.5 & 43.3 \\
            SubgraphRAG + Llama3.1-8B
            & 75.5 & 91.4 & 65.9 & 83.6 & 51.5 & 63.1 & 57.5 & 68.9 & 41.9 & 47.4 \\
            SubgraphRAG + GPT-4o-mini 
            & 80.6 & 92.9 & 74.1 & 88.5 & 57.4 & 67.3 & \tb{63.9} & \tb{72.7} & 51.1 & 54.4 \\
            \midrule
            Ours + Llama3.1-8B
            & 78.0 & 90.5 & 71.3 & 86.1 & 55.6 & 62.8 & 52.1 & 62.8 & 51.2 & 55.6 \\
            Ours + GPT-4o-mini
            & \tb{83.9} & \underline{95.6} & 75.3 & 88.5 & 60.6 & 69.7 & 59.0 & 70.8 & 53.5 & 57.9 \\
            \midrule
            Ours + Llama3.1-8B
            & 80.1 & 90.8 & 72.9 & 87.8 & 60.4 & 70.2 & 58.6 & \underline{71.8} & 54.0 & 59.6 \\
            Ours + GPT-4o-mini
            & \underline{82.7} & \tb{95.8} & \underline{76.3} & \underline{89.3} & \underline{61.8} & \underline{70.6} & 60.0 & 71.0 & 56.9 & 61.4 \\
            \midrule
            Ours + Llama3.1-8B
            & 81.1 & 91.4 & 74.0 & 88.1 & 61.4 & 70.3 & 59.9 & 71.5 & \underline{58.8} & \underline{63.2} \\
            Ours + GPT-4o-mini
            & \tb{83.9} & \tb{95.8} & \tb{77.4} & \tb{90.2} & \tb{62.0} & \tb{72.2} & \underline{60.8} & 71.6 & \tb{59.4} & \tb{64.7} \\
            \bottomrule
        \end{tabular}
    \end{adjustbox}
    \vspace{-10pt}
\end{table}

\tb{Can Path-Formatted Inputs Facilitate LLM Reasoning?}
We investigate whether path-based inputs can facilitate the reasoning efficacy of LLMs. We conduct experiments using retrieved texts with $K=120$ and $M=180$. For each sample, we convert the retrieved path into a set of triples and then randomly shuffle these triples. This procedure preserves the overall information content but removes structural coherence, requiring the model to reconstruct reasoning chains from 
disordered facts, posing a greater challenge for multi-hop questions. We then evaluate on the CWQ dataset, which contains more multi-hop reasoning samples than WebQSP. As shown in Figure~\ref{fig:path_format_effect}, both GPT-4o and GPT-4o-mini exhibit significant drops in Macro-F1 when using unstructured triple inputs. This result indicates that path-formatted inputs consistently benefit LLM-based reasoning across models with varying reasoning capacities, justifying the path-based reasoning paradigm in \proj.
\subsection{Empirical Analysis on Faithfulness and Generalization}

\definecolor{upgreen}{HTML}{A8D8B9}   
\definecolor{downred}{HTML}{F6A6A7}   
\begin{figure}[t]
  \centering
  \begin{subfigure}[c]{0.4\linewidth}
    \includegraphics[width=\linewidth]{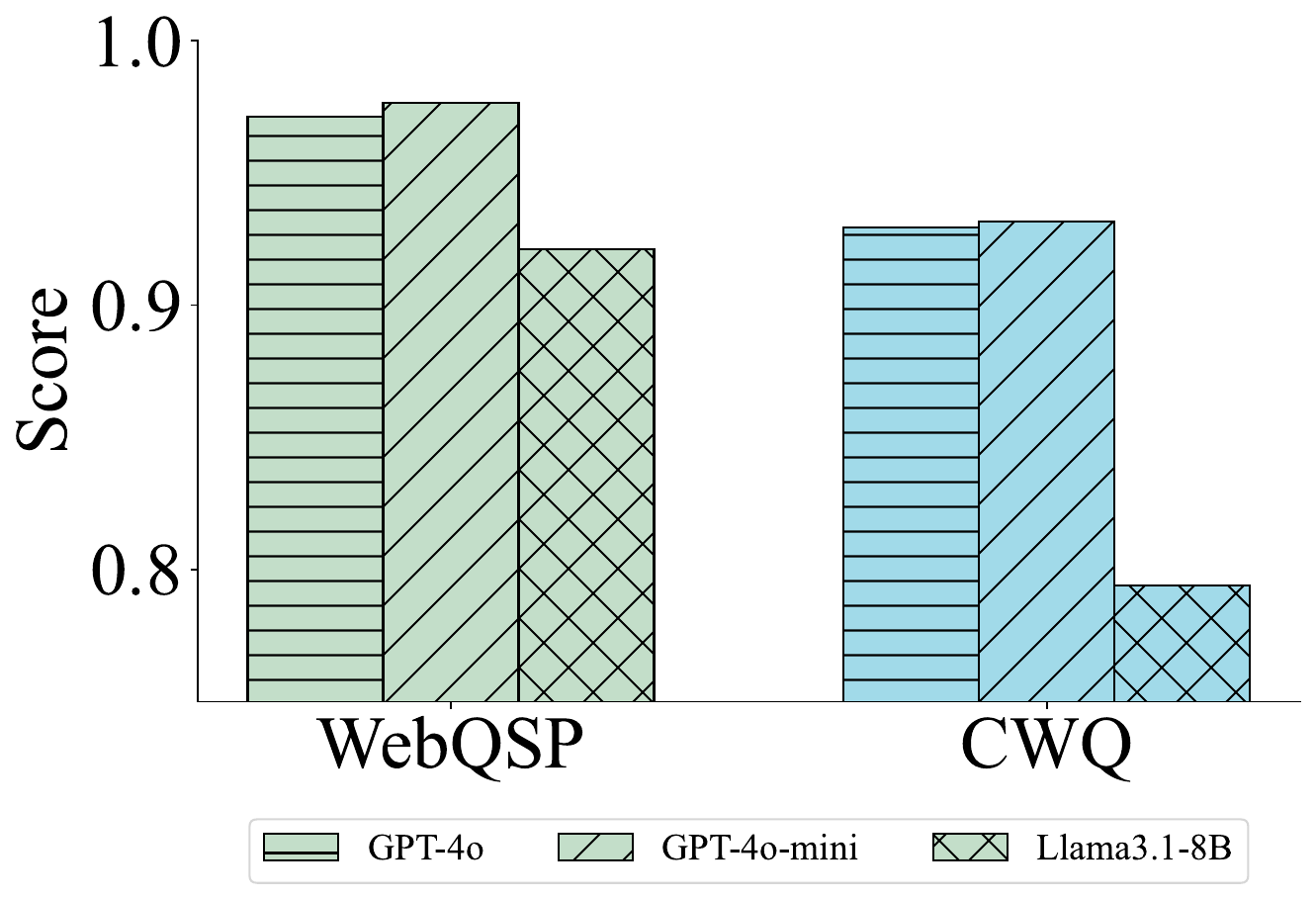}
    \caption{Hallucination scores on WebQSP and CWQ
             ($K{=}80$, $M{=}120$).}
    \label{fig:hallucination}
  \end{subfigure}\hfill
  \begin{subtable}[c]{0.56\linewidth}
    \caption{Generalization performance of \proj. $\leftrightarrow$ denotes train on one dataset while test on the other. We set $K=120,M=180$ and $K=800,M=1200$ respectively during inference time.}
    \label{tab:transferability}
    \vspace{0.4em}
    \scalebox{0.72}{%
        \begin{tabular}{@{}ccccc@{}}
        \toprule
        & \multicolumn{2}{c}{\textbf{WebQSP}} & \multicolumn{2}{c}{\textbf{CWQ}}\\
        \cmidrule(l){2-5}
        & Macro‑F1 & Hit & Macro‑F1 & Hit\\
        \midrule
        SubgraphRAG (100)                      & 77.5 & 90.1 & 54.1 & 62.0 \\
        SubgraphRAG ($\leftrightarrow$, 100)   & 73.8\,{\scriptsize\textcolor{downred}{$\downarrow$4.8\%}}
                                                & 88.1\,{\scriptsize\textcolor{downred}{$\downarrow$2.2\%}}
                                                & 44.7\,{\scriptsize\textcolor{downred}{$\downarrow$17.4\%}}
                                                & 54.2\,{\scriptsize\textcolor{downred}{$\downarrow$12.6\%}} \\
        SubgraphRAG ($\leftrightarrow$, 500)   & 76.2\,{\scriptsize\textcolor{downred}{$\downarrow$1.7\%}}
                                                & 91.2\,{\scriptsize\textcolor{upgreen}{$\uparrow$1.2\%}}
                                                & 50.3\,{\scriptsize\textcolor{downred}{$\downarrow$7.0\%}}
                                                & 60.8\,{\scriptsize\textcolor{downred}{$\downarrow$1.9\%}} \\
        \midrule
        Ours (41.5, 52.4)                      & 80.4 & 92.5 & 58.1 & 69.3 \\
        Ours ($\leftrightarrow$, 52.8, 57.3)   & 79.2\,{\scriptsize\textcolor{downred}{$\downarrow$1.5\%}}
                                                & 92.6\,{\scriptsize\textcolor{upgreen}{$\uparrow$0.1\%}}
                                                & 56.3\,{\scriptsize\textcolor{downred}{$\downarrow$3.1\%}}
                                                & 67.7\,{\scriptsize\textcolor{downred}{$\downarrow$2.3\%}} \\
        Ours ($\leftrightarrow$, 88.7, 107.2)  & 80.1\,{\scriptsize\textcolor{downred}{$\downarrow$0.4\%}}
                                                & 93.1\,{\scriptsize\textcolor{upgreen}{$\uparrow$0.6\%}}
                                                & 56.1\,{\scriptsize\textcolor{downred}{$\downarrow$3.4\%}}
                                                & 67.9\,{\scriptsize\textcolor{downred}{$\downarrow$2.0\%}} \\
        \bottomrule
        \end{tabular}
    }
    \vspace{0.8em}
  \end{subtable}
  \caption{Illustration of faithfulness and generalization performance.}
  \vspace{-1ex}
\end{figure}

\tb{Faithfulness.} We further evaluate whether our framework relies on retrieved knowledge or intrinsic parametric knowledge for question answering. To this end, we design a prompt to evaluate whether the answer was generated based 
on retrieved information for each sample with Hit > 0 using GPT-4o-mini. We 
define the average of the resulting binary outputs as the \textit{hallucination score}, where for a single sample a score of 1 indicates reliance on retrieved knowledge, and 0 otherwise. As shown in Figure~\ref{fig:hallucination}, both GPT-4o and GPT-4o-mini achieve hallucination scores $\geq 0.9$ on both datasets, indicating that the vast majority of predictions are grounded in retrieved information. On the other hand, we observe that all models are more prone to hallucinations on the CWQ dataset, likely due to its greater proportion of questions requiring more complex reasoning process. 

\tb{Generalizability.} 
In this section, we evaluate the generalization capability of our framework by training the graph retriever on one dataset and testing it on another. As shown in Table~\ref{tab:transferability}, the performance gap between in-distribution and cross-dataset settings for \proj{} is consistently smaller than that of SubgraphRAG. Notably, on the CWQ dataset, the model trained on WebQSP using \proj{} experiences only a 3.1\% drop in Macro-F1, compared to a 17.4\% drop observed for SubgraphRAG. These results demonstrate that our lightweight graph retriever, enhanced through design choices specifically tailored to the structure of knowledge graphs, attains strong generalization performance without requiring a (far) larger-scale LLM-based retriever.

\section{Conclusion}
We propose \proj{}, a lightweight graph retriever specifically designed for knowledge graphs and the KGQA task. By integrating several key components such as label rationalizer, model-agnostic graph transformation, and bidirectional message passing, \proj{} enhances representational capacity and improves generalization performance. Empirically, \proj{} achieves superior generalization ability, outperforming prior state-of-the-art methods by $2.66\%-20.34\%$ when paired with GPT-4o-mini and Llama3.1-8B. Notably, \proj{} significantly reduces the performance gap between smaller and more powerful LLM-based reasoners, as well as the gap under cross-dataset settings, highlighting its superior retrieval capability and generalizability.

\bibliographystyle{plain}
{\small
\bibliography{reference}

\begin{thebibliography}{10}

\bibitem{achiam2023gpt}
Josh Achiam, Steven Adler, Sandhini Agarwal, Lama Ahmad, Ilge Akkaya, Florencia~Leoni Aleman, Diogo Almeida, Janko Altenschmidt, Sam Altman, Shyamal Anadkat, et~al.
\newblock Gpt-4 technical report.
\newblock {\em arXiv preprint arXiv:2303.08774}, 2023.

\bibitem{berant2013semantic}
Jonathan Berant, Andrew Chou, Roy Frostig, and Percy Liang.
\newblock Semantic parsing on freebase from question-answer pairs.
\newblock In {\em Proceedings of the 2013 conference on empirical methods in natural language processing}, pages 1533--1544, 2013.

\bibitem{bollacker2008freebase}
Kurt Bollacker, Colin Evans, Praveen Paritosh, Tim Sturge, and Jamie Taylor.
\newblock Freebase: a collaboratively created graph database for structuring human knowledge.
\newblock In {\em Proceedings of the 2008 ACM SIGMOD international conference on Management of data}, pages 1247--1250, 2008.

\bibitem{brown2020language}
Tom Brown, Benjamin Mann, Nick Ryder, Melanie Subbiah, Jared~D Kaplan, Prafulla Dhariwal, Arvind Neelakantan, Pranav Shyam, Girish Sastry, Amanda Askell, et~al.
\newblock Language models are few-shot learners.
\newblock {\em Advances in Neural Information Processing Systems}, 33:1877--1901, 2020.

\bibitem{bubeck2023sparks}
Sébastien Bubeck, Varun Chandrasekaran, Ronen Eldan, Johannes Gehrke, Eric Horvitz, Ece Kamar, Peter Lee, Yin~Tat Lee, Yuanzhi Li, Scott Lundberg, Harsha Nori, Hamid Palangi, Marco~Tulio Ribeiro, and Yi~Zhang.
\newblock Sparks of artificial general intelligence: Early experiments with gpt-4.
\newblock {\em arXiv preprint arXiv:2303.12712}, 2023.

\bibitem{chein2008graph}
Michel Chein and Marie-Laure Mugnier.
\newblock {\em Graph-based knowledge representation: computational foundations of conceptual graphs}.
\newblock Springer Science \& Business Media, 2008.

\bibitem{chen2024plan}
Liyi Chen, Panrong Tong, Zhongming Jin, Ying Sun, Jieping Ye, and Hui Xiong.
\newblock Plan-on-graph: Self-correcting adaptive planning of large language model on knowledge graphs.
\newblock {\em arXiv preprint arXiv:2410.23875}, 2024.

\bibitem{dehghan-etal-2024-ewek}
Mohammad Dehghan, Mohammad Alomrani, Sunyam Bagga, David Alfonso-Hermelo, Khalil Bibi, Abbas Ghaddar, Yingxue Zhang, Xiaoguang Li, Jianye Hao, Qun Liu, Jimmy Lin, Boxing Chen, Prasanna Parthasarathi, Mahdi Biparva, and Mehdi Rezagholizadeh.
\newblock {EWEK}-{QA} : Enhanced web and efficient knowledge graph retrieval for citation-based question answering systems.
\newblock In Lun-Wei Ku, Andre Martins, and Vivek Srikumar, editors, {\em Proceedings of the 62nd Annual Meeting of the Association for Computational Linguistics (Volume 1: Long Papers)}, pages 14169--14187, Bangkok, Thailand, August 2024. Association for Computational Linguistics.

\bibitem{fan2024hardmath}
Jingxuan Fan, Sarah Martinson, Erik~Y Wang, Kaylie Hausknecht, Jonah Brenner, Danxian Liu, Nianli Peng, Corey Wang, and Michael~P Brenner.
\newblock Hardmath: A benchmark dataset for challenging problems in applied mathematics.
\newblock {\em arXiv preprint arXiv:2410.09988}, 2024.

\bibitem{fey2019fast}
Matthias Fey and Jan~Eric Lenssen.
\newblock Fast graph representation learning with pytorch geometric.
\newblock {\em arXiv preprint arXiv:1903.02428}, 2019.

\bibitem{galkin2023towards}
Mikhail Galkin, Xinyu Yuan, Hesham Mostafa, Jian Tang, and Zhaocheng Zhu.
\newblock Towards foundation models for knowledge graph reasoning.
\newblock {\em arXiv preprint arXiv:2310.04562}, 2023.

\bibitem{gao2023double}
Jianfei Gao, Yangze Zhou, and Bruno Ribeiro.
\newblock Double permutation equivariance for knowledge graph completion.
\newblock {\em arXiv preprint arXiv:2302.01313}, 2023.

\bibitem{gao2024two}
Yifu Gao, Linbo Qiao, Zhigang Kan, Zhihua Wen, Yongquan He, and Dongsheng Li.
\newblock Two-stage generative question answering on temporal knowledge graph using large language models.
\newblock {\em arXiv preprint arXiv:2402.16568}, 2024.

\bibitem{gao2024retrievalaugmented}
Yunfan Gao, Yun Xiong, Xinyu Gao, Kangxiang Jia, Jinliu Pan, Yuxi Bi, Yi~Dai, Jiawei Sun, Meng Wang, and Haofen Wang.
\newblock Retrieval-augmented generation for large language models: A survey.
\newblock {\em arXiv preprint arXiv:2312.10997}, 2024.

\bibitem{geng2023relational}
Yuxia Geng, Jiaoyan Chen, Jeff~Z Pan, Mingyang Chen, Song Jiang, Wen Zhang, and Huajun Chen.
\newblock Relational message passing for fully inductive knowledge graph completion.
\newblock In {\em 2023 IEEE 39th international conference on data engineering (ICDE)}, pages 1221--1233. IEEE, 2023.

\bibitem{hagberg2008exploring}
Aric Hagberg, Pieter~J Swart, and Daniel~A Schult.
\newblock Exploring network structure, dynamics, and function using networkx.
\newblock Technical report, Los Alamos National Laboratory (LANL), Los Alamos, NM (United States), 2008.

\bibitem{hamilton2017inductive}
Will Hamilton, Zhitao Ying, and Jure Leskovec.
\newblock Inductive representation learning on large graphs.
\newblock {\em Advances in neural information processing systems}, 30, 2017.

\bibitem{he2024g}
Xiaoxin He, Yijun Tian, Yifei Sun, Nitesh Chawla, Thomas Laurent, Yann LeCun, Xavier Bresson, and Bryan Hooi.
\newblock G-retriever: Retrieval-augmented generation for textual graph understanding and question answering.
\newblock {\em Advances in Neural Information Processing Systems}, 37:132876--132907, 2024.

\bibitem{hogan2021knowledge}
Aidan Hogan, Eva Blomqvist, Michael Cochez, Claudia d’Amato, Gerard~De Melo, Claudio Gutierrez, Sabrina Kirrane, Jos{\'e} Emilio~Labra Gayo, Roberto Navigli, Sebastian Neumaier, et~al.
\newblock Knowledge graphs.
\newblock {\em ACM Computing Surveys (Csur)}, 54(4):1--37, 2021.

\bibitem{huang2023towards}
Jie Huang and Kevin Chen-Chuan Chang.
\newblock Towards reasoning in large language models: A survey.
\newblock In {\em Findings of the Association for Computational Linguistics: ACL 2023}, pages 1049--1065, 2023.

\bibitem{huang2023a}
Lei Huang, Weijiang Yu, Weitao Ma, Weihong Zhong, Zhangyin Feng, Haotian Wang, Qianglong Chen, Weihua Peng, Xiaocheng Feng, Bing Qin, and Ting Liu.
\newblock A survey on hallucination in large language models: Principles, taxonomy, challenges, and open questions.
\newblock {\em arXiv preprint arXiv:2311.05232}, 2023.

\bibitem{ioffe2015batch}
Sergey Ioffe and Christian Szegedy.
\newblock Batch normalization: Accelerating deep network training by reducing internal covariate shift.
\newblock In {\em International conference on machine learning}, pages 448--456. pmlr, 2015.

\bibitem{10.1145/3571730}
Ziwei Ji, Nayeon Lee, Rita Frieske, Tiezheng Yu, Dan Su, Yan Xu, Etsuko Ishii, Ye~Jin Bang, Andrea Madotto, and Pascale Fung.
\newblock Survey of hallucination in natural language generation.
\newblock {\em ACM Computing Surveys}, 55(12), 2023.

\bibitem{jiang2023structgpt}
Jinhao Jiang, Kun Zhou, Zican Dong, Keming Ye, Wayne~Xin Zhao, and Ji-Rong Wen.
\newblock Structgpt: A general framework for large language model to reason over structured data.
\newblock In {\em Proceedings of the 2023 Conference on Empirical Methods in Natural Language Processing}, pages 9237--9251, 2023.

\bibitem{jiang2024kg}
Jinhao Jiang, Kun Zhou, Wayne~Xin Zhao, Yang Song, Chen Zhu, Hengshu Zhu, and Ji-Rong Wen.
\newblock Kg-agent: An efficient autonomous agent framework for complex reasoning over knowledge graph.
\newblock {\em arXiv preprint arXiv:2402.11163}, 2024.

\bibitem{jiang2022unikgqa}
Jinhao Jiang, Kun Zhou, Xin Zhao, and Ji-Rong Wen.
\newblock Unikgqa: Unified retrieval and reasoning for solving multi-hop question answering over knowledge graph.
\newblock In {\em The Eleventh International Conference on Learning Representations}, 2022.

\bibitem{jin2024graph}
Bowen Jin, Chulin Xie, Jiawei Zhang, Kashob~Kumar Roy, Yu~Zhang, Zheng Li, Ruirui Li, Xianfeng Tang, Suhang Wang, Yu~Meng, et~al.
\newblock Graph chain-of-thought: Augmenting large language models by reasoning on graphs.
\newblock {\em arXiv preprint arXiv:2404.07103}, 2024.

\bibitem{kasai2023realtime}
Jungo Kasai, Keisuke Sakaguchi, yoichi takahashi, Ronan~Le Bras, Akari Asai, Xinyan~Velocity Yu, Dragomir Radev, Noah~A. Smith, Yejin Choi, and Kentaro Inui.
\newblock Realtime {QA}: What's the answer right now?
\newblock In {\em Thirty-seventh Conference on Neural Information Processing Systems Datasets and Benchmarks Track}, 2023.

\bibitem{kim-etal-2023-kg}
Jiho Kim, Yeonsu Kwon, Yohan Jo, and Edward Choi.
\newblock {KG}-{GPT}: A general framework for reasoning on knowledge graphs using large language models.
\newblock In Houda Bouamor, Juan Pino, and Kalika Bali, editors, {\em Findings of the Association for Computational Linguistics: EMNLP 2023}, pages 9410--9421, Singapore, December 2023. Association for Computational Linguistics.

\bibitem{kingma2014adam}
Diederik~P Kingma and Jimmy Ba.
\newblock Adam: A method for stochastic optimization.
\newblock {\em arXiv preprint arXiv:1412.6980}, 2014.

\bibitem{kipf2016semi}
Thomas~N Kipf and Max Welling.
\newblock Semi-supervised classification with graph convolutional networks.
\newblock {\em arXiv preprint arXiv:1609.02907}, 2016.

\bibitem{lee2023ingram}
Jaejun Lee, Chanyoung Chung, and Joyce~Jiyoung Whang.
\newblock Ingram: Inductive knowledge graph embedding via relation graphs.
\newblock In {\em International Conference on Machine Learning}, pages 18796--18809. PMLR, 2023.

\bibitem{lewis2020retrieval}
Patrick Lewis, Ethan Perez, Aleksandra Piktus, Fabio Petroni, Vladimir Karpukhin, Naman Goyal, Heinrich K{\"u}ttler, Mike Lewis, Wen-tau Yih, Tim Rockt{\"a}schel, et~al.
\newblock Retrieval-augmented generation for knowledge-intensive nlp tasks.
\newblock {\em Advances in neural information processing systems}, 33:9459--9474, 2020.

\bibitem{li2024simple}
Mufei Li, Siqi Miao, and Pan Li.
\newblock Simple is effective: The roles of graphs and large language models in knowledge-graph-based retrieval-augmented generation.
\newblock {\em arXiv preprint arXiv:2410.20724}, 2024.

\bibitem{li2020distance}
Pan Li, Yanbang Wang, Hongwei Wang, and Jure Leskovec.
\newblock Distance encoding: Design provably more powerful neural networks for graph representation learning.
\newblock {\em Advances in Neural Information Processing Systems}, 33:4465--4478, 2020.

\bibitem{li2023graph}
Shiyang Li, Yifan Gao, Haoming Jiang, Qingyu Yin, Zheng Li, Xifeng Yan, Chao Zhang, and Bing Yin.
\newblock Graph reasoning for question answering with triplet retrieval.
\newblock In {\em Findings of the Association for Computational Linguistics: ACL 2023}, pages 3366--3375, 2023.

\bibitem{li2023towards}
Zehan Li, Xin Zhang, Yanzhao Zhang, Dingkun Long, Pengjun Xie, and Meishan Zhang.
\newblock Towards general text embeddings with multi-stage contrastive learning.
\newblock {\em arXiv preprint arXiv:2308.03281}, 2023.

\bibitem{liu2024explore}
Guangyi Liu, Yongqi Zhang, Yong Li, and Quanming Yao.
\newblock Explore then determine: A gnn-llm synergy framework for reasoning over knowledge graph.
\newblock {\em arXiv preprint arXiv:2406.01145}, 2024.

\bibitem{liu2025dualreasoninggnnllmcollaborative}
Guangyi Liu, Yongqi Zhang, Yong Li, and Quanming Yao.
\newblock Dual reasoning: A gnn-llm collaborative framework for knowledge graph question answering, 2025.

\bibitem{luo2024rog}
Linhao Luo, Yuan-Fang Li, Gholamreza Haffari, and Shirui Pan.
\newblock Reasoning on graphs: Faithful and interpretable large language model reasoning.
\newblock In {\em International Conference on Learning Representations}, 2024.

\bibitem{ma2024think}
Shengjie Ma, Chengjin Xu, Xuhui Jiang, Muzhi Li, Huaren Qu, and Jian Guo.
\newblock Think-on-graph 2.0: Deep and interpretable large language model reasoning with knowledge graph-guided retrieval.
\newblock {\em arXiv e-prints}, pages arXiv--2407, 2024.

\bibitem{manning2024automated}
Benjamin~S Manning, Kehang Zhu, and John~J Horton.
\newblock Automated social science: Language models as scientist and subjects.
\newblock Technical report, National Bureau of Economic Research, 2024.

\bibitem{mavromatis2022rearev}
Costas Mavromatis and George Karypis.
\newblock Rearev: Adaptive reasoning for question answering over knowledge graphs.
\newblock In {\em Findings of the Association for Computational Linguistics: EMNLP 2022}, pages 2447--2458, 2022.

\bibitem{mavromatis2024gnn}
Costas Mavromatis and George Karypis.
\newblock Gnn-rag: Graph neural retrieval for large language model reasoning.
\newblock {\em arXiv preprint arXiv:2405.20139}, 2024.

\bibitem{llama3}
Meta.
\newblock Build the future of ai with meta llama 3, 2024.

\bibitem{mikolov2013efficient}
Tomas Mikolov, Kai Chen, Greg Corrado, and Jeffrey Dean.
\newblock Efficient estimation of word representations in vector space.
\newblock {\em arXiv preprint arXiv:1301.3781}, 2013.

\bibitem{chatgpt}
OpenAI.
\newblock Introducing chatgpt, 2022.

\bibitem{gpt4o}
OpenAI.
\newblock Hello gpt-4o, 2024.

\bibitem{paszke2019pytorchimperativestylehighperformance}
Adam Paszke, Sam Gross, Francisco Massa, Adam Lerer, James Bradbury, Gregory Chanan, Trevor Killeen, Zeming Lin, Natalia Gimelshein, Luca Antiga, Alban Desmaison, Andreas Köpf, Edward Yang, Zach DeVito, Martin Raison, Alykhan Tejani, Sasank Chilamkurthy, Benoit Steiner, Lu~Fang, Junjie Bai, and Soumith Chintala.
\newblock Pytorch: An imperative style, high-performance deep learning library, 2019.

\bibitem{robinson2015graph}
Ian Robinson, Jim Webber, and Emil Eifrem.
\newblock {\em Graph databases: new opportunities for connected data}.
\newblock " O'Reilly Media, Inc.", 2015.

\bibitem{shuster-etal-2021-retrieval-augmentation}
Kurt Shuster, Spencer Poff, Moya Chen, Douwe Kiela, and Jason Weston.
\newblock Retrieval augmentation reduces hallucination in conversation.
\newblock In {\em Findings of the Association for Computational Linguistics: EMNLP 2021}, pages 3784--3803, 2021.

\bibitem{srivastava2014dropout}
Nitish Srivastava, Geoffrey Hinton, Alex Krizhevsky, Ilya Sutskever, and Ruslan Salakhutdinov.
\newblock Dropout: a simple way to prevent neural networks from overfitting.
\newblock {\em The journal of machine learning research}, 15(1):1929--1958, 2014.

\bibitem{sunthink}
Jiashuo Sun, Chengjin Xu, Lumingyuan Tang, Saizhuo Wang, Chen Lin, Yeyun Gong, Lionel Ni, Heung-Yeung Shum, and Jian Guo.
\newblock Think-on-graph: Deep and responsible reasoning of large language model on knowledge graph.
\newblock In {\em The Twelfth International Conference on Learning Representations}, 2024.

\bibitem{talmor2018web}
Alon Talmor and Jonathan Berant.
\newblock The web as a knowledge-base for answering complex questions.
\newblock In {\em Proceedings of the 2018 Conference of the North American Chapter of the Association for Computational Linguistics: Human Language Technologies, Volume 1 (Long Papers)}, pages 641--651, 2018.

\bibitem{touvron2023llama}
Hugo Touvron, Louis Martin, Kevin Stone, Peter Albert, Amjad Almahairi, Yasmine Babaei, Nikolay Bashlykov, Soumya Batra, Prajjwal Bhargava, Shruti Bhosale, et~al.
\newblock Llama 2: Open foundation and fine-tuned chat models.
\newblock {\em arXiv preprint arXiv:2307.09288}, 2023.

\bibitem{velivckovic2017graph}
Petar Veli{\v{c}}kovi{\'c}, Guillem Cucurull, Arantxa Casanova, Adriana Romero, Pietro Lio, and Yoshua Bengio.
\newblock Graph attention networks.
\newblock {\em arXiv preprint arXiv:1710.10903}, 2017.

\bibitem{velickovic2017graph}
Petar Velickovic, Guillem Cucurull, Arantxa Casanova, Adriana Romero, Pietro Lio, Yoshua Bengio, et~al.
\newblock Graph attention networks.
\newblock {\em stat}, 1050(20):10--48550, 2017.

\bibitem{wang2023knowledge}
Keheng Wang, Feiyu Duan, Sirui Wang, Peiguang Li, Yunsen Xian, Chuantao Yin, Wenge Rong, and Zhang Xiong.
\newblock Knowledge-driven cot: Exploring faithful reasoning in llms for knowledge-intensive question answering.
\newblock {\em arXiv preprint arXiv:2308.13259}, 2023.

\bibitem{wang2024knowledge}
Yu~Wang, Nedim Lipka, Ryan~A Rossi, Alexa Siu, Ruiyi Zhang, and Tyler Derr.
\newblock Knowledge graph prompting for multi-document question answering.
\newblock In {\em Proceedings of the AAAI Conference on Artificial Intelligence}, volume~38, pages 19206--19214, 2024.

\bibitem{wei2022chain}
Jason Wei, Xuezhi Wang, Dale Schuurmans, Maarten Bosma, brian ichter, Fei Xia, Ed~H. Chi, Quoc~V Le, and Denny Zhou.
\newblock Chain of thought prompting elicits reasoning in large language models.
\newblock In {\em Advances in Neural Information Processing Systems}, 2022.

\bibitem{wen2023mindmap}
Yilin Wen, Zifeng Wang, and Jimeng Sun.
\newblock Mindmap: Knowledge graph prompting sparks graph of thoughts in large language models.
\newblock {\em arXiv preprint arXiv:2308.09729}, 2023.

\bibitem{wu2024survey}
Likang Wu, Zhi Zheng, Zhaopeng Qiu, Hao Wang, Hongchao Gu, Tingjia Shen, Chuan Qin, Chen Zhu, Hengshu Zhu, Qi~Liu, et~al.
\newblock A survey on large language models for recommendation.
\newblock {\em World Wide Web}, 27(5):60, 2024.

\bibitem{wu2023retrieverewriteanswerkgtotextenhancedllms}
Yike Wu, Nan Hu, Sheng Bi, Guilin Qi, Jie Ren, Anhuan Xie, and Wei Song.
\newblock Retrieve-rewrite-answer: A kg-to-text enhanced llms framework for knowledge graph question answering, 2023.

\bibitem{xu2018how}
Keyulu Xu, Weihua Hu, Jure Leskovec, and Stefanie Jegelka.
\newblock How powerful are graph neural networks?
\newblock In {\em International Conference on Learning Representations}, 2019.

\bibitem{qwen2}
An~Yang, Baosong Yang, Binyuan Hui, Bo~Zheng, Bowen Yu, Chang Zhou, Chengpeng Li, Chengyuan Li, Dayiheng Liu, Fei Huang, Guanting Dong, Haoran Wei, Huan Lin, Jialong Tang, Jialin Wang, Jian Yang, Jianhong Tu, Jianwei Zhang, Jianxin Ma, Jin Xu, Jingren Zhou, Jinze Bai, Jinzheng He, Junyang Lin, Kai Dang, Keming Lu, Keqin Chen, Kexin Yang, Mei Li, Mingfeng Xue, Na~Ni, Pei Zhang, Peng Wang, Ru~Peng, Rui Men, Ruize Gao, Runji Lin, Shijie Wang, Shuai Bai, Sinan Tan, Tianhang Zhu, Tianhao Li, Tianyu Liu, Wenbin Ge, Xiaodong Deng, Xiaohuan Zhou, Xingzhang Ren, Xinyu Zhang, Xipin Wei, Xuancheng Ren, Yang Fan, Yang Yao, Yichang Zhang, Yu~Wan, Yunfei Chu, Yuqiong Liu, Zeyu Cui, Zhenru Zhang, and Zhihao Fan.
\newblock Qwen2 technical report.
\newblock {\em arXiv preprint arXiv:2407.10671}, 2024.

\bibitem{yao2024tree}
Shunyu Yao, Dian Yu, Jeffrey Zhao, Izhak Shafran, Tom Griffiths, Yuan Cao, and Karthik Narasimhan.
\newblock Tree of thoughts: Deliberate problem solving with large language models.
\newblock {\em Advances in Neural Information Processing Systems}, 36, 2024.

\bibitem{yih2016value}
Wen-tau Yih, Matthew Richardson, Christopher Meek, Ming-Wei Chang, and Jina Suh.
\newblock The value of semantic parse labeling for knowledge base question answering.
\newblock In {\em Proceedings of the 54th Annual Meeting of the Association for Computational Linguistics (Volume 2: Short Papers)}, pages 201--206, 2016.

\bibitem{zhang2022subgraph}
Jing Zhang, Xiaokang Zhang, Jifan Yu, Jian Tang, Jie Tang, Cuiping Li, and Hong Chen.
\newblock Subgraph retrieval enhanced model for multi-hop knowledge base question answering.
\newblock In {\em Proceedings of the 60th Annual Meeting of the Association for Computational Linguistics (Volume 1: Long Papers)}, pages 5773--5784, 2022.

\bibitem{zhang2021labeling}
Muhan Zhang, Pan Li, Yinglong Xia, Kai Wang, and Long Jin.
\newblock Labeling trick: A theory of using graph neural networks for multi-node representation learning.
\newblock {\em Advances in Neural Information Processing Systems}, 34:9061--9073, 2021.

\bibitem{zhou2023ood}
Jincheng Zhou, Beatrice Bevilacqua, and Bruno Ribeiro.
\newblock An ood multi-task perspective for link prediction with new relation types and nodes.
\newblock {\em arXiv preprint arXiv:2307.06046}, 23, 2023.

\end{thebibliography}
}
\appendix
\newpage
\begin{center}
	\LARGE \bf {Appendix}
\end{center}

\etocdepthtag.toc{mtappendix}
\etocsettagdepth{mtchapter}{none}
\etocsettagdepth{mtappendix}{subsection}
\tableofcontents
\newpage

\section{Broad Impact}
This work explores the integration of KGs with LLMs for the task of KGQA, with a focus on improving retrieval quality and generalization through graph-based retrievers. By grounding LLMs in structured relational knowledge and employing a more interpretable and generalizable retrieval model, our proposed framework aims to significantly mitigate hallucination and improve factual consistency.

The potential societal benefits of this research are considerable. More faithful and robust KG-augmented language systems could enhance the reliability of AI assistants in high-stakes domains such as scientific research, healthcare, and legal reasoning—where factual accuracy and interpretability are critical. In particular, reducing hallucinations in LLM outputs can support safer deployment of AI systems in real-world applications.
Moreover, our method is efficient at inference stage, and does not rely on large-scale supervised training or fine-tuning of LLMs, making it more accessible and sustainable, especially in resource-constrained environments. Additionally, by supporting smaller LLMs with strong retrieval capabilities, our work helps democratize access to high-performance question answering systems without requiring access to proprietary or computationally expensive language models for knowledge retrieval.

\section{Related Work}
\label{related_work}
\tb{Retrieve-then-reasoning paradigm}.
In KG-based RAG, a substantial number of methods adopt the \textit{retrieve-then-reasoning} paradigm~\cite{li2023graph,kim-etal-2023-kg,liu2025dualreasoninggnnllmcollaborative,wu2023retrieverewriteanswerkgtotextenhancedllms,wen2023mindmap,mavromatis2024gnn,li2024simple,luo2024rog,mavromatis2022rearev}, wherein a retriever first extracts relevant triples from the knowledge graph, followed by a (LLM-based) that generates the final answer based on the retrieved information. The retriever can be broadly categorized into LLM-based and graph-based approaches.
LLM-based retrievers benefit from large model capacity and strong semantic understanding. However, they also suffer from hallucination issues, and high computational cost and latency. To address this, recent work has proposed lightweight GNN-based graph retrievers~\cite{mavromatis2022rearev,li2024simple,zhang2022subgraph,mavromatis2022rearev} that operate directly on the KG structure. These methods have achieved superior performance on KGQA benchmarks, benefiting from the strong reasoning and denoising capabilities of downstream LLM-based reasoners. While graph-based retrievers offer substantial computational advantages and mitigating hallucinations, they still face generalization challenges~\cite{li2024simple}. In this work, we aim to retain the efficiency of graph-based retrievers while  enhancing their expressivity and generalization ability through a model-agnostic design tailored for characteristics of knowledge graphs and KGQA task.

\tb{KG-based agentic RAG}. 
Another line of research leverages LLMs as agents that iteratively explore the knowledge graph to retrieve relevant information~\cite{gao2024two,wang2024knowledge,jiang2024kg,sunthink,chen2024plan,ma2024think,jin2024graph}. In this setting, the agent integrates both retrieval and reasoning capabilities, enabling more adaptive knowledge access. While this approach has demonstrated effectiveness in identifying relevant triples, the iterative exploration process incurs substantial latency, as well as computational costs due to repeated LLM calls. In contrast, our method adopts a lightweight graph-based retriever to balance efficiency and effectiveness for KGQA.

\section{Complexity Analysis}
\label{app:complexity}
\tb{Preprocessing.} Given a graph $\mcal{G} = (\mcal{V}, \mcal{E})$ with $|\mcal{V}|$ nodes and $|\mcal{E}|$ edges,  The time complexity of this transformation is $\mathcal{O}(|\mcal{E}| d_{\max})$, where $d_{\max}$ is the maximum node degree in $\mcal{G}$. The space complexity is $\mathcal{O}(|\mcal{E}| + |\mcal{E}'|)$, where $|\mcal{E}'|$ denotes the number of edges in the resulting line graph $\mcal{G}'$, typically on the order of $\mathcal{O}(|E| d_{avg})$, with $d_{avg}$ as the average node degree.

\tb{Model training and inference.} For both training and inference, with a $K$-layer GCN operating on the line graph $\mcal{G}_q'$, the time complexity is $\mathcal{O}(K|\mcal{E}'| F)$, where $F$ is the dimensionality of node features. The space complexity is $\mathcal{O}(|\mcal{V}'|F + |\mcal{E}'|)$, where $|V'|$ is the number of triples (i.e., nodes in the line graph). During model training and inference, it does not involve any LLM call for the retriever. 

\section{More Discussions on Line Graph Transformation}
\label{discuss_graph_trasnform}
In this section, we discuss the limitations of additional message-passing GNNs when applied to the original knowledge graph, and elaborate on how these limitations can be addressed through line graph transformation.

\tb{GraphSage~\cite{hamilton2017inductive}}. The upadte function of GraphSage is shown in Eqn.~\ref{eq:sage_edge_concat}. 
\begin{equation}\label{eq:sage_edge_concat}
  \mathbf m_{ij}^{(k)}
    = W_{\text{msg}}^{(k)}
      \bigl[\,
        \mathbf h_j^{(k-1)} \,\Vert\, \mathbf r_{ij}
      \bigr],
  \qquad
  \mathbf h_i^{(k)}
    = \sigma\!\Bigl(
        W_{\text{self}}^{(k)}\mathbf h_i^{(k-1)}
        \;+\;
        \mathrm{MEAN}\!\bigl\{
          \mathbf m_{ij}^{(k)}
        \bigr\}
      \Bigr).
\end{equation}
As seen, only the entity states $\mathbf h_i^{(k)}$ are upadted, and relation embeddings $\mathbf r_{ij}$ remain static, so the model cannot refine relation semantics nor capture cross-triple interactions. 
After converting to the line graph $\mathcal G'_q$, the propagation is performed between triples, enabling explicit inter-triple reasoning.
\begin{equation}\label{eq:sage_line_concat}
  \mathbf h_u^{(k)}
    = \sigma\!\Bigl(
        W_{\text{self}}^{(k)}
        \,\phi\bigl(\mathbf e_u,\mathbf r_u,\mathbf e_u'\bigr)
        \;+\;
        W_{\text{nbr}}^{(k)}
        \frac{1}{|\mathcal N_\ell(u)|}
        \sum_{v\in\mathcal N_\ell(u)}
          \phi\bigl(\mathbf e_v,\mathbf r_v,\mathbf e_v'\bigr)
      \Bigr),
\end{equation}

\tb{Graph Attention Network (GAT)~\cite{velivckovic2017graph}}. The upadte function of GAT is shown in Eqn.~\ref{eq:gat_edge_concat}. Similarly $\mathbf r_{ij}$ only modulates one-hop attention, therefore the relational embedding is not updated, only the entity embedding is learned. 
\begin{align}\label{eq:gat_edge_concat}
  \alpha_{ij}^{(k)}
    &= \mathrm{softmax}_{j}\!
       \Bigl(
         a^\top
         \bigl[
           W \mathbf h_i^{(k-1)} \,\Vert\,
           W \mathbf h_j^{(k-1)} \,\Vert\,
           W_e \mathbf r_{ij}
         \bigr]
       \Bigr), \nonumber\\
  \mathbf h_i^{(k)}
    &= \sigma\!\Bigl(
        \sum_{j\in\mathcal N(i)}
          \alpha_{ij}^{(k)}\;
          V\bigl[\,
            \mathbf h_j^{(k-1)} \,\Vert\, \mathbf r_{ij}
          \bigr]
      \Bigr).
\end{align}
After the line-graph transform, Attention is now computed between triples, so both intra-triple composition and inter-triple dependencies influence edge-aware attention.

\begin{equation}
  \alpha_{ij}
  \;=\;
  \operatorname{softmax}_{h_j}
  \!\bigl(
    a^\top\!
    \bigl[
      W\phi(\mathbf e_i,\mathbf r_i,\mathbf e_i')
      \;\Vert\;
      W\phi(\mathbf e_j,\mathbf r_j,\mathbf e_j')
    \bigr]
  \bigr),
  \qquad
  \rvh_{i}^{(k)}
  \;=\;
  \sigma\!\Bigl(
    \sum_{h_j\in\mathcal N_\ell(h_i)}
      \alpha_{ij}\,
      W\phi(\mathbf e_j,\mathbf r_j,\mathbf e_j')
  \Bigr).
\end{equation}
Similarly, for all message-passing GNNs that follow the general \textit{message–aggregation–update} paradigm (Eq.~\ref{eq:general_gnn}), the relational structure in the original knowledge (sub)graph cannot be fully exploited. In contrast, the line graph transformation offers a model-agnostic solution that enables richer modeling of both intra- and inter-triple interactions, facilitating more expressive relational reasoning. This transformation is broadly applicable across a wide range of GNN architectures.
\begin{equation}
\begin{aligned}
\label{eq:general_gnn}
\mathbf m_{ij}^{(k)}
  &= \mathcal M^{(k)}
     \!\bigl(
       \mathbf h_i^{(k-1)},\;
       \mathbf h_j^{(k-1)},\;
       \mathbf r_{ij}
     \bigr),
     &&\forall\, j\!\in\!\mathcal N(i),
\\[4pt]
\mathbf a_i^{(k)}
  &= \mathcal A^{(k)}
     \!\Bigl(
       \bigl\{\mathbf m_{ij}^{(k)} : j\!\in\!\mathcal N(i)\bigr\}
     \Bigr),
\\[4pt]
\mathbf h_i^{(k)}
  &= \mathcal U^{(k)}
     \!\bigl(
       \mathbf h_i^{(k-1)},\;
       \mathbf a_i^{(k)}
     \bigr),
\end{aligned}
\end{equation}

\section{Algorithmic Pseudocode}
\label{pseudocode}
The overall preprocessing and training procedure is illustrated in Algorithm~\ref{alg:proj}.

\begin{algorithm}[t]
\caption{Overall Procedure of \proj}
\label{alg:proj}
\begin{algorithmic}[1]
\Require Training set $\mathcal{D}=\{(q,e_q,e_a,\mathcal{G}_q)\}$, epochs $E$, learning rate $\eta$, hyper‑parameters $\lambda_q,\lambda_{path}$
\Ensure  Optimized retriever parameters $\theta=\{\overrightarrow{\theta},\overleftarrow{\theta}\}$ and $\phi$
\Statex\textbf{Preprocessing}
\ForAll{$(q,e_q,e_a,\mathcal{G}_q)\in\mathcal{D}$}
    \State Construct line graph $\mathcal{G}'_q \gets \textsc{LineGraph}(\mathcal{G}_q)$
    \State Generate candidate paths $\mathcal{P}_{cand} \gets \textsc{GenPaths}(\mathcal{G}_q,d_{\min},d_{\min}{+}2)$
    \State Rationalize labels $\mathcal{Y}_q \gets \gamma(q,\mathcal{P}_{cand})$ \Comment{LLM call}
    \State Relation targeting $\mathcal{R}_* \gets \gamma(q,\textsc{Relations}(\mathcal{G}_q))$ \Comment{LLM call}
\EndFor
\Statex\textbf{Initialization}
\State Initialize bi‑directional GCN encoders $f_{\overrightarrow{\theta}},\,f_{\overleftarrow{\theta}}$ and STOP MLP $g_\phi$
\Statex\textbf{Training}
\For{$e=1$ \textbf{to} $E$}
    \ForAll{minibatch $\mathcal{B}\subset\mathcal{D}$}
        \ForAll{$(q,\mathcal{G}'_q,\mathcal{Y}_q,\mathcal{R}_*)\in\mathcal{B}$}
            \State Encode nodes $\mathbf{z}_i \gets 
            \tfrac12\!\bigl(f_{\overrightarrow{\theta}}(v_i)+
            f_{\overleftarrow{\theta}}(v_i)\bigr)$
            \State Build $\mathcal{V}_{pos},\mathcal{V}_{neg}$ using $\mathcal{R}_*$   
            \State $\mathcal{L}_q \gets \textsc{NegSampleLoss}(\theta,\phi;\mathcal{V}_{pos},\mathcal{V}_{neg})$ \Comment{Eq.~\eqref{eq:L_q_neg_sampling}}
            \State $\mathcal{L}_{path} \gets \textsc{PathLoss}(\theta,\phi;q,\mathcal{Y}_q)$  \Comment{Eq.~\eqref{path_loss}}
        \EndFor
        \State $\mathcal{L} \gets \lambda_q\mathcal{L}_q + \lambda_{path}\mathcal{L}_{path}$ \Comment{$\lambda_q$, $\lambda_{path}$ default to $1.0$}
        \State Update $\theta,\phi$ via Adam with step size $\eta$
    \EndFor
\EndFor
\end{algorithmic}
\end{algorithm}

\section{Theoretical Properties of the Directed Line Graph}
Let $\mathcal{G}=(\mcal{V},\mcal{E})$ be a finite directed graph.
Its directed line graph is denoted by $l(\mathcal{G})=(\mcal{V}_{\ell},\mcal{E}_{\ell})$,
where every node $x\in\mcal{V}_{\ell}$ corresponds to a directed edge
$e_x=(u,v)\in\mcal{E}$, and there is an edge $(x,y)\in\mcal{E}_{\ell}$
iff the head of $e_x$ equals the tail of $e_y$.

\vspace{4pt}
\begin{proposition}[Bijective mapping of directed paths]
\label{prop:path-bijection}
Let
\[
P \;=\;
\bigl(
  v_0 \xrightarrow{\,e_1\,} v_1 \xrightarrow{\,e_2\,} \dots
  \xrightarrow{\,e_k\,} v_k
\bigr),\quad k\ge 1,
\]
be a directed path of length $k$ in $\mathcal{G}$, where
$e_i=(v_{\,i-1},v_i)\in\mcal{E}$.
Define
\[
\Phi(P)
\;=\;
\bigl(e_1 \xrightarrow{} e_2 \xrightarrow{}\dots\xrightarrow{} e_k\bigr),
\]
regarding each $e_i$ as its corresponding node in $l(\mathcal{G})$.
Then $\Phi$ is a bijection between the set of directed paths of length $k$ in $\mathcal{G}$, and
the set of directed paths of length $k-1$ in $l(\mathcal{G})$.
\end{proposition}

\begin{proof}
\emph{Well‑definedness.}  
For consecutive edges $e_i=(v_{i-1},v_i)$ and $e_{i+1}=(v_i,v_{i+1})$
the shared endpoint is $v_i$ with consistent orientation, hence
$(e_i,e_{i+1})\in\mcal{E}_{\ell}$.
Thus $\Phi(P)$ is a valid path of length $k-1$ in~$l(\mathcal{G})$.

\emph{Injectivity.}  
If two original paths $P\neq P'$ differ at position $j$, then
$\Phi(P)$ and $\Phi(P')$ differ at node~$j$, so $\Phi(P)\neq\Phi(P')$.

\emph{Surjectivity.}  
Take any path $(x_1\!\to x_2\!\to\dots\to x_k)$ in $l(\mathcal{G})$ and
let $e_i$ be the edge in $\mcal{E}$ represented by $x_i$.
Because $(x_i,x_{i+1})\in\mcal{E}_{\ell}$, the head of $e_i$ is the tail
of $e_{i+1}$.  Hence $(e_1,\dots,e_k)$ forms a length‑$k$ path in
$\mathcal{G}$ whose image under~$\Phi$ is the given path.
\end{proof}

\vspace{4pt}
\begin{proposition}[Path‑length reduction]
\label{prop:length-reduction}
Let $u,v\in\mcal{V}$ be connected in $\mathcal{G}$ by a shortest
directed path of length $d\ge 1$,
\(
u=v_0\!\xrightarrow{e_1}\!\dots\xrightarrow{e_d} v=v_d.
\)
Let $x_u$ and $x_v$ be the nodes of $\mcal{V}_{\ell}$ corresponding to
the first edge $e_1$ and the last edge $e_d$, respectively.
Then the distance between $x_u$ and $x_v$ in $l(\mathcal{G})$ is
exactly $d-1$, and no shorter path exists in $l(\mathcal{G})$.
\end{proposition}

\begin{proof}
By Proposition~\ref{prop:path-bijection}, the path
$(e_1,e_2,\dots,e_d)$ is a directed path of length $d-1$ from
$x_u$ to $x_v$ in $l(\mathcal{G})$.
Assume for contradiction that there is a shorter path of length
$\ell<d-1$ between $x_u$ and $x_v$ in $\mcal{E}_{\ell}$.
The surjectivity of $\Phi$ then yields a directed path of length
$\ell+1<d$ from $u$ to $v$ in $\mathcal{G}$, contradicting the
minimality of $d$.
Hence the distance in $l(\mathcal{G})$ equals $d-1$.
\end{proof}

\noindent
Propositions~\ref{prop:path-bijection} confirms that there is a one-to-one mapping of reasoning paths between $l(\mcal{G})$ and $\mcal{G}$, therefore we can perform path-based learning and inference in the line graph. Propositions~\ref{prop:length-reduction} implies that a $K$-layer GNN model in $l(\mcal{G})$ is equivalent in the receptive field as a $(K+1)$-layer GNN model in graph $G$, therefore the receptive field of a GNN model increases via line graph transformation.

\section{Datasets}
\label{dataset}
WebQSP is a benchmark dataset for KGQA, derived from the original WebQuestions dataset~\cite{berant2013semantic}. It comprises 4,737 natural language questions annotated with full semantic parses in the form of SPARQL queries executable against Freebase. The dataset emphasizes single-hop questions, typically involving a direct relation between the question and answer entities.

CWQ dataset extends the WebQSP dataset to address more challenging multi-hop question answering scenarios. It contains 34,689 complex questions that require reasoning over multiple facts and relations. Each question is paired with a SPARQL query and corresponding answers, facilitating evaluation in both semantic parsing and information retrieval contexts. The datasets statistics can be found in Table~\ref{tab:dataset_stats}. 

\begin{table}[!t]
\centering
\caption{Dataset statistics and distribution of answer set sizes.}
\label{tab:dataset_stats}
\begin{adjustbox}{width=0.9\textwidth}
\begin{tabular}{c|cc|cccc}
\toprule
\multirow{2}{*}{Dataset} & \multicolumn{2}{c|}{Dataset Size} & \multicolumn{4}{c}{Distribution of Answer Set Size} \\
\cmidrule{2-7}
 & \#Train & \#Test & $\#\text{Ans}=1$ & $2 \leq \#\text{Ans} \leq 4$ & $5 \leq \#\text{Ans} \leq 9$ & $\#\text{Ans} \geq 10$ \\
\midrule
WebQSP & 2,826 & 1,628 & 51.2\% & 27.4\% & 8.3\% & 12.1\% \\
CWQ    & 27,639 & 3,531 & 70.6\% & 19.4\% & 6.0\% & 4.0\% \\
\bottomrule
\end{tabular}
\end{adjustbox}
\end{table}

Following previous practice, we adopt the same training and test split, with the same subgraph construction for each question-answer pair to ensure fairness~\cite{jiang2022unikgqa,luo2024rog,li2024simple,mavromatis2024gnn}.

\section{More Details on Experimental Setup and Implementations}
\label{app:exp}
\tb{Setup.} For model training, we employ two 2-layer GCNs to enable bidirectional message passing. Each GCN has a hidden dimension of 512. We use the Adam optimizer~\cite{kingma2014adam} with a learning rate of $1\times10^{-3}$, and a batch size of 10. Batch normalization~\cite{ioffe2015batch} is not used, as we observe gradient instability when it is applied. The graph retriever is trained for 15 epochs on both datasets, and model selection is performed using cross-validation based on the validation loss. A dropout rate of 0.2~\cite{srivastava2014dropout} is applied for regularization. When multiple valid paths are available, we randomly sample one as the ground-truth supervision signal at each training step.

For evaluating KGQA performance on the test sets, we first generate answers using the downstream reasoner, followed by answer verification using GPT-4o-mini. To prevent potential information leakage, we decouple the answer generation and verification processes, avoiding inclusion of ground-truth answers in the generation prompts. For answer verification, we adopt chain-of-thought prompting~\cite{wei2022chain} to ensure accurate estimation of Macro-F1 and Hit metrics.

\tb{Implementation.} We utilize \ti{networkx}\cite{hagberg2008exploring} for performing line graph transformations and explore all paths between question entities (source nodes) and answer entities (target nodes), and GPT-4o-mini is used during preprocessing for relation targeting. Our remaining implementations are based on PyTorch\cite{paszke2019pytorchimperativestylehighperformance} and PyTorch Geometric~\cite{fey2019fast}.

\section{Efficiency Analysis}
We evaluate the efficiency of the proposed method and baselines based on three metrics: average runtime, average $\#$ LLM calls, and average $\#$ retrieved triples. As shown in Table~\ref{tab:efficiency}, agentic RAG methods (e.g., ToG) incur significantly higher latency and computational cost due to repeated LLM invocations. Among KG-based RAG methods, approaches employing LLM-based retrievers generally require more time than graph-based retrievers due to the LLM inference. Although GNN-RAG and SubgraphRAG exhibit comparable runtime and LLM calls to \proj, our method is more effective thanks to its design choices specifically tailored for KGQA. Furthermore, \proj{} retrieves no more than 50 triples, benefiting from the path-based inference approach, which allows \proj{} to learn when to stop reasoning, thereby avoiding the need to recall a fixed number of triples  as in SubgraphRAG. As a result, \proj{} enables more efficient downstream reasoning with reduced computes, balancing effiency and effectiveness. 
\begin{table}[!t]
\centering
\caption{Efficiency analysis of different methods on WebQSP dataset.}
\label{tab:efficiency}
\scalebox{0.9}{
\begin{tabular}{@{}c|c|c|c|c@{}}

\toprule
\textbf{Methods}                 & \textbf{Hit}               & \textbf{Avg. Runtime (s)} & \textbf{Avg. \# LLM Calls} & \textbf{Avg. \# Triples} \\ \midrule
RoG     & 85.6 & 8.65  & 2    & 49  \\
ToG     & 75.1 & 19.03 & 13.2 & 410 \\
GNN-RAG & 85.7 & 1.82  & 1    & 27  \\
\multicolumn{1}{l|}{SubgraphRAG} & \multicolumn{1}{l|}{90.1} & 2.63                      & 1                          & 100                      \\ \midrule
Ours    & 92.0 & 2.16  & 1    & 32  \\ \bottomrule
\end{tabular}}
\end{table}

\section{Motivating Examples on Rational Paths}
\label{app:label-rationale}
 In this section, we provide 6 intuitive examples of the claim for each case: \tb{(i)} Multiple shortest paths may exist for a given question, not all of which are semantically meaningful, and \tb{(ii)} Some causally-grounded paths may not be the shortest. Figure~\ref{fig:type1_no1}-\ref{fig:type1_no6} demonstrates supports for case 1, and Figure~\ref{fig:type2_no1}-\ref{fig:type2_no6} demonstrates supports for case 2. 

\begin{figure}[t]
\begin{tcolorbox}[colback=gray!5!white,
                  colframe=gray!75!black,
                  title=WebQTest-923\_e3a2d3d50bac69d563de83a7f72eafc0]
\scriptsize

\textbf{Question:}\\
Which country with religious organization leadership \emph{Noddfa, Treorchy} borders England?

\rule{\linewidth}{0.4pt}

\textbf{Candidate shortest paths:}\\[2pt]
{\raggedright   
\texttt{\textbf{England\patharrow location.location.adjoin\_s\patharrow m.04dgsfb\patharrow
location.adjoining\_relationship.adjoins\patharrow Wales}}\hfill
\textcolor{red}{(rational)}\\[6pt]

\texttt{England\patharrow law.court\_jurisdiction\_area.courts\patharrow
National Industrial Relations Court\patharrow law.court.jurisdiction\patharrow Wales}\\[6pt]

\texttt{England\patharrow organization.organization\_scope.organizations\_with\_this\_scope\patharrow
Police Federation of England and Wales\patharrow
organization.organization.geographic\_scope\patharrow Wales}\\[6pt]

\texttt{England\patharrow organization.organization\_scope.organizations\_with\_this\_scope\patharrow
BES Utilities\patharrow organization.organization.geographic\_scope\patharrow Wales}\\[6pt]

\ldots 
\par}

\rule{\linewidth}{0.4pt}

\textbf{Explanation:}\\[2pt]
{\raggedright
\textbf{The first path:} It encodes a direct geographical‐adjacency relation (\texttt{location.location.adjoin\_s}
followed by \texttt{location.adjoining\_relationship.adjoins}), so it
correctly captures that Wales \emph{borders} England.\par\vspace{6pt}

\textbf{The second path:} It links England and Wales through a shared court
system, reflecting \emph{legal jurisdiction} rather than physical
contiguity; therefore it is not a rational answer.\par\vspace{6pt}

\textbf{The third path:} It relies on an organisation (Police Federation of
England and Wales) that operates in both regions.  Operational scope
signals administrative overlap, not territorial borders.\par\vspace{6pt}

\textbf{The fourth path:} Like the third, it uses an organisation’s
geographic scope (BES Utilities) to connect the two regions, so it conveys
no information about adjacency and is likewise non‑rational.\par}

\end{tcolorbox}
\caption{Motivating example on not all shortest paths are rational paths.}
\label{fig:type1_no1}
\end{figure}

\begin{figure}[t]
\begin{tcolorbox}[colback=gray!5!white,
                  colframe=gray!75!black,
                  title=WebQTest-415\_b6ad66a3f1f515d0688c346e16d202e6]
\scriptsize

\textbf{Question:}\\
What movie with film character named Mr. Woodson did Tupac star in?

\rule{\linewidth}{0.4pt}

\textbf{Candidate shortest paths:}\\[2pt]
{\raggedright
\texttt{\textbf{Tupac Shakur\patharrow film.actor.film\patharrow m.0jz0c4\patharrow film.performance.film\patharrow Gridlock'd}}\hfill
\textcolor{red}{(rational)}\\[6pt]

\texttt{Tupac Shakur\patharrow music.featured\_artist.recordings\patharrow Out The Moon\patharrow music.recording.releases\patharrow Gridlock'd}\\[6pt]

\texttt{Tupac Shakur\patharrow music.featured\_artist.recordings\patharrow Wanted Dead or Alive \\ \patharrow music.recording.releases\patharrow Gridlock'd}\\[6pt]

\texttt{Tupac Shakur\patharrow music.artist.track\_contributions\patharrow m.0nj8wrw\patharrow music.track\_contribution.track \\ \patharrow Out The Moon\patharrow music.recording.releases\patharrow Gridlock'd}\\[6pt]

\texttt{Tupac Shakur\patharrow film.music\_contributor.film\patharrow Def Jam's How to Be a Player\patharrow film.film.produced\_by \\ \patharrow Russell Simmons\patharrow film.producer.films\_executive\_produced\patharrow Gridlock'd}\\[6pt]

\par}

\rule{\linewidth}{0.4pt}

\textbf{Explanation:}\\[2pt]
{\raggedright
\textbf{The first path:} This path follows the relation \texttt{film.actor.film} from Tupac Shakur to a role entity, and then \texttt{film.performance.film} to the film Gridlock'd. It accurately models the actor–character–film linkage, so it is a rational answer.\par\vspace{6pt}

\textbf{The second path:} It connects Tupac Shakur to the film Gridlock'd via a featured music recording. This reflects musical involvement, not acting or character presence, so it does not address the question.\par\vspace{6pt}

\textbf{The third path:} Similar to the second, it identifies Tupac as a featured artist on a song associated with the film. However, musical contributions do not imply he played a film character.\par\vspace{6pt}

\textbf{The fourth path:} This path involves nested musical metadata, eventually reaching Gridlock'd via track and recording releases. It does not establish that Tupac portrayed a character in the movie, hence it is not rational.\par\vspace{6pt}

\textbf{The fifth path:} It connects Tupac to Gridlock'd via his contribution to another film (Def Jam's How to Be a Player), which shares a producer with Gridlock'd. This is an indirect production-based relation, not evidence of his acting in Gridlock'd.\par}
\end{tcolorbox}
\caption{Motivating example on not all shortest paths are rational paths.}
\label{fig:type1_no2}
\end{figure}

\begin{figure}[t]
\begin{tcolorbox}[colback=gray!5!white,
                  colframe=gray!75!black,
                  title=WebQTrn-3696\_b874dcb19fa3a6c4e6037dc13f1f3bc4]
\scriptsize

\textbf{Question:}\\
Which state senator from Georgia took this position at the earliest date?

\rule{\linewidth}{0.4pt}

\textbf{Candidate shortest paths:}\\[2pt]
{\raggedright
\texttt{\textbf{Georgia\patharrow government.political\_district.representatives\patharrow m.030qq3n \\ \patharrow government.government\_position\_held.office\_holder\patharrow Saxby Chambliss}}\hfill
\textcolor{red}{(rational)}\\[6pt]

\texttt{Georgia\patharrow government.political\_district.elections\patharrow United States Senate election in Georgia, 2008\patharrow common.topic.image\patharrow Saxby Chambliss}\\[6pt]

\par}

\rule{\linewidth}{0.4pt}

\textbf{Explanation:}\\[2pt]
{\raggedright
\textbf{The first path:} This path directly follows \texttt{government.political\_district.representatives} from Georgia to a position held by an entity, and then uses \texttt{government.government\_position\_held.office\_holder} to identify the person (Saxby Chambliss) who held the position. This correctly models a government role held by someone representing the district, making it a rational and temporally grounded path for determining who assumed the position earliest.\par\vspace{6pt}

\textbf{The second path:} This path connects Georgia to a 2008 Senate election, and then to Saxby Chambliss via an image relation. Although the entity appears in the context of the election, the path does not encode any formal position-holding information or temporal precedence. It is therefore unrelated to identifying the office-holder or the date of assuming the position, and is non-rational.\par}

\end{tcolorbox}
\caption{Motivating example on not all shortest paths are rational paths.}
\label{fig:type1_no3}
\end{figure}

\begin{figure}[t]
\begin{tcolorbox}[colback=gray!5!white,
                  colframe=gray!75!black,
                  title=WebQTrn-3763\_c707414f103503f2530fc654a85645fe]
\scriptsize

\textbf{Question:}\\
What country close to Russia has a religious organization named Ukrainian Greek Catholic Church?

\rule{\linewidth}{0.4pt}

\textbf{Candidate shortest paths:}\\[2pt]
{\raggedright
\texttt{\textbf{Ukrainian Greek Catholic Church\patharrow religion.religious\_organization.leaders\patharrow m.05tnwqd \\ \patharrow religion.religious\_organization\_leadership.jurisdiction\patharrow Ukraine}}\hfill
\textcolor{red}{(rational)}\\[6pt]

\texttt{Russia\patharrow location.location.partially\_contains\patharrow Seym River\patharrow geography.river.basin\_countries\patharrow Ukraine}\\[6pt]

\texttt{Russia\patharrow olympics.olympic\_participating\_country.olympics\_participated\_in\patharrow 2010 Winter Olympics \\ \patharrow olympics.olympic\_games.participating\_countries\patharrow Ukraine}\\[6pt]

\texttt{Russia\patharrow organization.organization\_founder.organizations\_founded\patharrow Commonwealth of Independent States \\ \patharrow organization.organization.founders\patharrow Ukraine}\\[6pt]

\texttt{Russia\patharrow location.location.adjoin\_s\patharrow m.02wj9d3\patharrow location.adjoining\_relationship.adjoins\patharrow Ukraine}\\[6pt]

\par}

\rule{\linewidth}{0.4pt}

\textbf{Explanation:}\\[2pt]
{\raggedright
\textbf{The first path:} This path connects the Ukrainian Greek Catholic Church via \texttt{religion.religious\_organization.leaders} to a leadership entity, then via \texttt{religion.religious\_organization\_leadership.jurisdiction} to Ukraine. It correctly encodes the organizational jurisdiction of the church and identifies the relevant country, making it a rational answer.\par\vspace{6pt}

\textbf{The second path:} This connects Russia to the Seym River and then to Ukraine via river basin membership. It reflects geographic proximity but does not capture any religious organizational structure. Hence, it is non-rational.\par\vspace{6pt}

\textbf{The third path:} This path shows that both Russia and Ukraine participated in the same Olympic games. While it may imply contemporaneity or international context, it says nothing about religious institutions or jurisdictions.\par\vspace{6pt}

\textbf{The fourth path:} This connects Russia and Ukraine through the shared founding of the Commonwealth of Independent States. Although it reflects political cooperation, it does not reveal any information about religious affiliation or structure.\par\vspace{6pt}

\textbf{The fifth path:} This path shows Russia shares a border with Ukraine. It satisfies the "close to Russia" part of the question, but lacks any information about the Ukrainian Greek Catholic Church, making it non-rational.\par}

\end{tcolorbox}
\caption{Motivating example on not all shortest paths are rational paths.}
\label{fig:type1_no4}
\end{figure}

\begin{figure}[t]
\begin{tcolorbox}[colback=gray!5!white,
                  colframe=gray!75!black,
                  title=WebQTrn-3548\_c352f5de0efe2369ee74ef2a99973561]
\scriptsize

\textbf{Question:}\\
What city was the birthplace of Charlton Heston and a famous pro athlete who started their career in 2007?

\rule{\linewidth}{0.4pt}

\textbf{Candidate shortest paths:}\\[2pt]
{\raggedright
\texttt{\textbf{Charlton Heston\patharrow people.person.places\_lived\patharrow m.0h28vy2\patharrow people.place\_lived.location\patharrow Los Angeles}}\hfill
\textcolor{red}{(rational)}\\[6pt]

\texttt{Charlton Heston\patharrow people.deceased\_person.place\_of\_death\patharrow Beverly Hills \\ \patharrow base.biblioness.bibs\_location.city\patharrow Los Angeles}\\[6pt]

\texttt{Charlton Heston\patharrow people.person.children\patharrow Fraser Clarke Heston\patharrow people.person.place\_of\_birth \\ \patharrow Los Angeles}\\[6pt]

\ldots 
\par}

\rule{\linewidth}{0.4pt}

\textbf{Explanation:}\\[2pt]
{\raggedright
\textbf{The first path:} This path uses \texttt{people.person.places\_lived} to access a lived-location node and then follows \texttt{people.place\_lived.location} to reach Los Angeles. Since place-of-birth often overlaps with early-life residence, this path is rational in the absence of explicit birth data.\par\vspace{6pt}

\textbf{The second path:} This connects Charlton Heston to Los Angeles via his place of death (Beverly Hills), which is a sub-location of Los Angeles. However, death location is unrelated to birthplace, making this path non-rational.\par\vspace{6pt}

\textbf{The third path:} This path identifies Charlton Heston’s child and retrieves that child’s birthplace (Los Angeles). While it shares a location with the question's answer, it provides no evidence of Charlton Heston’s own birthplace and is therefore non-rational.\par}

\end{tcolorbox}
\caption{Motivating example on not all shortest paths are rational paths.}
\label{fig:type1_no5}
\end{figure}

\begin{figure}[t]
\begin{tcolorbox}[colback=gray!5!white,
                  colframe=gray!75!black,
                  title=WebQTrn-1399\_64fc62dc06d16e612aafb00889d4ada1]
\scriptsize

\textbf{Question:}\\
What is the country close to Russia where Mikheil Saakashvili holds a government position?

\rule{\linewidth}{0.4pt}

\textbf{Candidate shortest paths:}\\[2pt]
{\raggedright
\texttt{\textbf{Mikheil Saakashvili\patharrow government.politician.government\_positions\_held\patharrow m.0j6t55g \\ \patharrow government.government\_position\_held.jurisdiction\_of\_office\patharrow Georgia}}\hfill
\textcolor{red}{(rational)}\\[6pt]

\texttt{Mikheil Saakashvili\patharrow organization.organization\_founder.organizations\_founded \\ \patharrow United National Movement\patharrow organization.organization.geographic\_scope\patharrow Georgia}\\[6pt]

\texttt{Russia\patharrow location.location.partially\_contains\patharrow Diklosmta \\ \patharrow location.location.partially\_containedby\patharrow Georgia}\\[6pt]

\texttt{Russia\patharrow location.country.languages\_spoken\patharrow Osetin Language \\ \patharrow language.human\_language.main\_country\patharrow Georgia}\\[6pt]

\par}

\rule{\linewidth}{0.4pt}

\textbf{Explanation:}\\[2pt]
{\raggedright
\textbf{The first path:} This path follows \texttt{government.politician.government\_positions\_held} to retrieve a position node, and then uses \texttt{government.government\_position\_held.jurisdiction\_of\_office} to reach Georgia. This explicitly identifies the country where Mikheil Saakashvili held a government role, satisfying both the political and geographical parts of the question. It is thus rational.\par\vspace{6pt}

\textbf{The second path:} This path captures that Saakashvili founded an organization with activity in Georgia. While it indicates political involvement, it does not assert that he held a formal government position in the country. Hence, non-rational.\par\vspace{6pt}

\textbf{The third path:} This connects Russia to Georgia via a geographic relation involving Diklosmta. It reflects proximity, but does not involve Saakashvili or government roles—so it is irrelevant to the question.\par\vspace{6pt}

\textbf{The fourth path:} This path links Russia and Georgia through a shared spoken language (Osetin). While this indicates cultural or linguistic ties, it provides no information about Saakashvili’s political role. It is non-rational.\par\vspace{6pt}
\par}

\end{tcolorbox}
\caption{Motivating example on not all shortest paths are rational paths.}
\label{fig:type1_no6}
\end{figure}

\begin{figure}[t]
\begin{tcolorbox}[colback=gray!5!white,
                  colframe=gray!75!black,
                  title=WebQTrn-2946\_93c6dae3d218dbe112d4120f45c93298]
\scriptsize

\textbf{Question:}\\
What team with mascot named Champ did Tyson Chandler play for?

\rule{\linewidth}{0.4pt}

\textbf{Candidate shortest paths:}\\[2pt]
{\raggedright
\texttt{\textbf{Tyson Chandler\patharrow sports.pro\_athlete.teams\patharrow m.0j2jj7v\patharrow sports.sports\_team\_roster.team \\ \patharrow Dallas Mavericks}}\hfill
\textcolor{red}{(rational)}\\[6pt]

\texttt{Tyson Chandler\patharrow sports.pro\_athlete.teams\patharrow m.0110h779\patharrow sports.sports\_team\_roster.team \\ \patharrow Dallas Mavericks}\\[6pt]

\texttt{Champ\patharrow sports.mascot.team\patharrow Dallas Mavericks}\hfill
\textcolor{blue}{(shortest path)}\\[6pt]
\par}

\rule{\linewidth}{0.4pt}

\textbf{Explanation:}\\[2pt]
{\raggedright
\textbf{The first path:} This path uses \texttt{sports.pro\_athlete.teams} to retrieve a team membership record, and \texttt{sports.sports\_team\_roster.team} to reach the team (Dallas Mavericks). It directly connects Tyson Chandler to the team he played for, making it a rational path.\par\vspace{6pt}

\textbf{The second path:} This is structurally identical to the first but uses a different team-roster node. It also reaches Dallas Mavericks through the correct relation pair, so it is equally valid and rational in form—though redundant if the goal is to find \emph{a} team Chandler played for with mascot Champ.\par\vspace{6pt}

\textbf{The third path:} This connects the mascot Champ to the Dallas Mavericks. It identifies the team correctly but does not involve Tyson Chandler. Therefore, on its own, it does not answer the question. It is necessary context but not sufficient, and thus non-rational as a standalone path.\par}

\end{tcolorbox}
\caption{Motivating example on shortest paths may not be rational paths.}
\label{fig:type2_no1}
\end{figure}

\begin{figure}[t]
\begin{tcolorbox}[colback=gray!5!white,
                  colframe=gray!75!black,
                  title=WebQTrn-934\_02aae167a8fa9f7d45daab265ac650cd]
\scriptsize

\textbf{Question:}\\
Who held their governmental position from 1786 and was the British General of the Revolutionary War?

\rule{\linewidth}{0.4pt}

\textbf{Candidate shortest paths:}\\[2pt]
{\raggedright
\texttt{\textbf{Kingdom of Great Britain\patharrow military.military\_combatant.military\_commanders\patharrow m.04fttv1 \\ \patharrow military.military\_command.military\_commander\patharrow Charles Cornwallis, 1st Marquess Cornwallis}}\hfill
\textcolor{red}{(rational)}\\[6pt]

\texttt{American Revolutionary War\patharrow base.culturalevent.event.entity\_involved \\ \patharrow Charles Cornwallis, 1st Marquess Cornwallis}\hfill
\textcolor{blue}{(shortest)}
\\[6pt]
\ldots
\par}

\rule{\linewidth}{0.4pt}

\textbf{Explanation:}\\[2pt]
{\raggedright
\textbf{The first path:} This path begins with the Kingdom of Great Britain, follows \texttt{military.military\_combatant.military\_commanders} to a command structure, and then \texttt{military.military\_command.military\_commander} to Charles Cornwallis. It directly encodes his role as a British military commander, making it a rational path aligned with the question.\par\vspace{6pt}

\textbf{The second path:} This connects Charles Cornwallis to the American Revolutionary War via an \texttt{entity\_involved} relation. While it establishes that he was involved in the war, it does not specify a command role nor relate to the governmental position, so it is non-rational.
\par}

\end{tcolorbox}
\caption{Motivating example on shortest paths may not be rational paths.}
\label{fig:type2_no2}
\end{figure}

\begin{figure}[t]
\begin{tcolorbox}[colback=gray!5!white,
                  colframe=gray!75!black,
                  title=WebQTrn-2349\_e831da3802943dad506eb1e3fb611847]
\scriptsize

\textbf{Question:}\\
What are the official bird and flower of the state whose capital is Lansing?

\rule{\linewidth}{0.4pt}

\textbf{Candidate shortest paths:}\\[2pt]
{\raggedright
\texttt{\textbf{Lansing\patharrow base.biblioness.bibs\_location.state\patharrow Michigan \\ \patharrow government.governmental\_jurisdiction.official\_symbols\patharrow m.04st85s \\ \patharrow location.location\_symbol\_relationship.symbol\patharrow American robin}}\hfill
\textcolor{red}{(rational)}\\[6pt]

\texttt{State bird\patharrow location.offical\_symbol\_variety.symbols\_of\_this\_kind\patharrow m.04st83j \\ \patharrow location.location\_symbol\_relationship.symbol\patharrow American robin}\hfill
\textcolor{blue}{(shortest)}\\[6pt]

\texttt{State flower\patharrow location.offical\_symbol\_variety.symbols\_of\_this\_kind\patharrow m.0hz8zmz \\ \patharrow location.location\_symbol\_relationship.symbol\patharrow Apple Blossom}\hfill
\textcolor{blue}{(shortest)}\\[6pt]
\ldots
\par}

\rule{\linewidth}{0.4pt}

\textbf{Explanation:}\\[2pt]
{\raggedright
\textbf{The first path:} Begin at Lansing, proceed through \texttt{base.biblioness.bibs\_location.state} to reach Michigan, and then use \texttt{government.governmental\_jurisdiction.official\_symbols} to retrieve the state’s official bird and flower via \texttt{location.location\_symbol\_relationship.symbol}. It correctly model the semantic intent of the question and are rational.\par\vspace{6pt}

\textbf{The second path:} This begins from the general category “State bird” and navigates to the American robin, but it lacks a connection to Michigan or Lansing, so it cannot determine the relevant state’s identity. Thus, it is non-rational.\par\vspace{6pt}

\textbf{The third path:} Similar to the third, this links the symbolic category “State flower” to Apple Blossom, but it does not connect this symbol to any particular state. It is non-rational on its own.
\par}

\end{tcolorbox}
\caption{Motivating example on shortest paths may not be rational paths.}
\label{fig:type2_no3}
\end{figure}

\begin{figure}[t]
\begin{tcolorbox}[colback=gray!5!white,
                  colframe=gray!75!black,
                  title=WebQTrn-88\_54f1262a1dbcb7b82f5b8ebd614401b9]
\scriptsize

\textbf{Question:}\\
In what city and state is the university that publishes the newspaper titled \emph{Santa Clara}?

\rule{\linewidth}{0.4pt}

\textbf{Candidate shortest paths:}\\[2pt]
{\raggedright
\texttt{\textbf{Santa Clara\patharrow education.school\_newspaper.school\patharrow Santa Clara University\patharrow \\ location.location.containedby\patharrow Santa Clara}}\hfill
\textcolor{red}{(rational)}\\[6pt]

\texttt{\textbf{Santa Clara\patharrow education.school\_newspaper.school\patharrow Santa Clara University\patharrow \\ location.location.containedby\patharrow California}}\hfill
\textcolor{red}{(rational)}\\[6pt]

\texttt{Santa Clara\patharrow book.newspaper.circulation\_areas\patharrow California}\hfill
\textcolor{blue}{(shortest)}\\[6pt]

\texttt{Santa Clara\patharrow book.newspaper.headquarters\patharrow Santa Clara\patharrow \\ location.mailing\_address.state\_province\_region\patharrow California}\\[6pt]
\ldots
\par}

\rule{\linewidth}{0.4pt}

\textbf{Explanation:}\\[2pt]
{\raggedright
\textbf{The first and second paths:} These follow \texttt{education.school\_newspaper.school} to reach Santa Clara University, then use \texttt{location.location.containedby} to identify its city (Santa Clara) and state (California). These directly trace the geographic location of the university that publishes the newspaper, making both paths rational.\par\vspace{6pt}

\textbf{The third path:} This connects the newspaper \emph{Santa Clara} to California via a general \texttt{book.newspaper.circulation\_areas} relation. While it suggests distribution within California, it does not ground the newspaper in a specific institution or location, so it is non-rational.\par\vspace{6pt}

\textbf{The fourth path:} This path uses the newspaper's headquarters to reach a city and then derives the state from a mailing address field. It circumvents the university relation required by the question and therefore does not directly answer it; it is non-rational.\par}
\end{tcolorbox}
\caption{Motivating example on shortest paths may not be rational paths.}
\label{fig:type2_no4}
\end{figure}

\begin{figure}[t]
\begin{tcolorbox}[colback=gray!5!white,
                  colframe=gray!75!black,
                  title=WebQTrn-2591\_9f8fc8341d7c53fe16d94a6a23638ec4]
\scriptsize

\textbf{Question:}\\
What country using the Malagasy Ariary currency is China's trading partner?

\rule{\linewidth}{0.4pt}

\textbf{Candidate shortest paths:}\\[2pt]
{\raggedright
\texttt{\textbf{China\patharrow location.statistical\_region.places\_exported\_to\patharrow m.04bfg2f\patharrow \\  location.imports\_and\_exports.exported\_to\patharrow Madagascar}}\hfill
\textcolor{red}{(rational)}\\[6pt]

\texttt{Malagasy ariary\patharrow finance.currency.countries\_used\patharrow Madagascar}\hfill
\textcolor{blue}{(shortest)}\\[6pt]

\texttt{China\patharrow travel.travel\_destination.tour\_operators\patharrow Bunnik Tours\patharrow \\ travel.tour\_operator.travel\_destinations\patharrow Madagascar}\\[6pt]
\ldots
\par}

\rule{\linewidth}{0.4pt}

\textbf{Explanation:}\\[2pt]
{\raggedright
\textbf{The first path:} This path uses \texttt{location.statistical\_region.places\_exported\_to} followed by \texttt{location.imports\_and\_exports.exported\_to} to connect China to Madagascar through a trade relationship. It correctly captures the fact that Madagascar is one of China's trading partners, making it rational.\par\vspace{6pt}

\textbf{The second path:} This identifies Madagascar as a country that uses the Malagasy Ariary, but it does not involve China or any trade relationship. It is relevant as supporting context for currency, but not sufficient to answer the question on its own.\par\vspace{6pt}

\textbf{The third path:} This connects China to Madagascar via a tourism relationship involving a tour operator (Bunnik Tours). While it shows interaction between the countries, it does not concern trade and is thus non-rational in the context of the question.\par}

\end{tcolorbox}
\caption{Motivating example on shortest paths may not be rational paths.}
\label{fig:type2_no5}
\end{figure}

\begin{figure}[t]
\begin{tcolorbox}[colback=gray!5!white,
                  colframe=gray!75!black,
                  title=WebQTrn-2237\_723ad981dc68e3cfe82e7134c8ca8fdb]
\scriptsize

\textbf{Question:}\\
Where are some places to stay at in the city where Gavin Newsom is a government officer?

\rule{\linewidth}{0.4pt}

\textbf{Candidate shortest paths:}\\[2pt]
{\raggedright
\texttt{\textbf{Gavin Newsom\patharrow government.political\_appointer.appointees\patharrow m.03k0n\_d\patharrow \\ government.government\_position\_held.jurisdiction\_of\_office\patharrow San Francisco\patharrow \\ travel.travel\_destination.accommodation\patharrow W San Francisco}}\hfill
\textcolor{red}{(rational)}\\[6pt]

\texttt{Gavin Newsom\patharrow people.person.place\_of\_birth\patharrow San Francisco\patharrow \\ travel.travel\_destination.accommodation\patharrow W San Francisco}\hfill
\textcolor{blue}{(shortest)}\\[6pt]

\texttt{Gavin Newsom\patharrow people.person.place\_of\_birth\patharrow San Francisco\patharrow \\ travel.travel\_destination.accommodation\patharrow Hostelling International, City Center}\hfill
\textcolor{blue}{(shortest)}\\[6pt]
\ldots
\par}

\rule{\linewidth}{0.4pt}

\textbf{Explanation:}\\[2pt]
{\raggedright
\textbf{The first path:} This path links Gavin Newsom to San Francisco through a government appointee role and then uses \texttt{government.government\_position\_held.jurisdiction\_of\_office} to specify the city in which he held office. From there, it identifies accommodations such as W San Francisco via \texttt{travel.travel\_destination.accommodation}. It directly answers the question and is rational.\par\vspace{6pt}

\textbf{The second and third paths:} These use \texttt{people.person.place\_of\_birth} to link Gavin Newsom to San Francisco and then list accommodations in that city. While San Francisco is both his birthplace and his jurisdiction of office, these paths do not establish that the city is where he was a government officer. Thus, they are non-rational according to the expected relation chain.\par}
\end{tcolorbox}
\caption{Motivating example on shortest paths may not be rational paths.}
\label{fig:type2_no6}
\end{figure}

\section{Demonstrations on Retrieved Reasoning Paths from \proj}
In this section, we provide 10 examples to demonstrate the qualitative results of \proj, as shown in Figure~\ref{fig:example1}-\ref{fig:example10}.

\begin{figure}[t]
\begin{tcolorbox}[colback=gray!5!white,
                  colframe=gray!75!black,
                  title=WEBQSP-WebQTest-7]
\scriptsize

\textbf{Question:}\\
Where was George Washington Carver from?

\rule{\linewidth}{0.4pt}

\textbf{Retrieved Paths:}\\[2pt]
{\raggedright
\texttt{George Washington Carver\patharrow people.person.nationality\patharrow United States of America}\\[6pt]

\texttt{George Washington Carver\patharrow people.deceased\_person.place\_of\_death\patharrow Tuskegee}\\[6pt]

\texttt{George Washington Carver\patharrow people.person.places\_lived\patharrow m.03prs0h}\\[6pt]

\texttt{\textbf{George Washington Carver\patharrow people.person.place\_of\_birth\patharrow Diamond}}\\[6pt]

\texttt{George Washington Carver\patharrow people.person.places\_lived\patharrow m.03ppx0s}\\[6pt]

\texttt{George Washington Carver\patharrow people.person.education\patharrow m.04hdfv4}\\[6pt]

\texttt{George Washington Carver\patharrow people.person.education\patharrow m.04hdfv4\patharrow \\ education.education.institution\patharrow Iowa State University}\\[6pt]

\texttt{George Washington Carver\patharrow people.person.places\_lived\patharrow m.03ppx0s\patharrow \\ people.place\_lived.location\patharrow Tuskegee}\\[6pt]

\texttt{George Washington Carver\patharrow people.person.places\_lived\patharrow m.03prs0h\patharrow \\ people.place\_lived.location\patharrow Joplin}\\[6pt]

\texttt{George Washington Carver\patharrow people.person.place\_of\_birth\patharrow Diamond\patharrow \\ location.statistical\_region.population\patharrow m.0hlfnly}\\[6pt]

\par}

\rule{\linewidth}{0.4pt}

\textbf{Ground-truth:}\\
Diamond

\end{tcolorbox}
\caption{Example on the retrieved reasoning paths by \proj}
\label{fig:example1}
\end{figure}

\begin{figure}[t]
\begin{tcolorbox}[colback=gray!5!white,
                  colframe=gray!75!black,
                  title=WEBQSP-WebQTest-928]
\scriptsize

\textbf{Question:}\\
What colleges did Harper Lee attend?

\rule{\linewidth}{0.4pt}

\textbf{Retrieved Paths:}\\[2pt]
{\raggedright
\texttt{Harper Lee\patharrow people.person.education\patharrow m.0lwxmyl\patharrow \\  education.education.institution\patharrow Monroe County High School}\\[6pt]

\texttt{Harper Lee\patharrow people.person.education\patharrow m.0lwxmy1\patharrow \\  education.education.institution\patharrow Huntingdon College}\\[6pt]

\texttt{Harper Lee\patharrow people.person.education\patharrow m.0lwxmy9\patharrow  \\ education.education.institution\patharrow University of Oxford}\\[6pt]

\texttt{Harper Lee\patharrow people.person.education\patharrow m.0n1l46h\patharrow \\  education.education.institution\patharrow University of Alabama School of Law}\\[6pt]

\texttt{Harper Lee\patharrow people.person.education\patharrow m.04hx138\patharrow \\  education.education.institution\patharrow University of Alabama}\\[6pt]

\par}

\rule{\linewidth}{0.4pt}

\textbf{Ground-truth:}\\
University of Alabama, Huntingdon College, University of Oxford, University of Alabama School of Law

\end{tcolorbox}
\caption{Example on the retrieved reasoning paths by \proj}
\label{fig:example2}
\end{figure}

\begin{figure}[t]
\begin{tcolorbox}[colback=gray!5!white,
                  colframe=gray!75!black,
                  title=WEBQSP-WebQTest-1205]
\scriptsize

\textbf{Question:}\\
Who plays Harley Quinn?

\rule{\linewidth}{0.4pt}

\textbf{Retrieved Paths:}\\[2pt]
{\raggedright
\texttt{Harley Quinn\patharrow tv.tv\_character.appeared\_in\_tv\_program\patharrow m.02wm17r\patharrow \\ tv.regular\_tv\_appearance.actor\patharrow Hynden Walch}\\[6pt]

\texttt{Harley Quinn\patharrow cvg.game\_character.games\patharrow m.09dycc\_\patharrow \\ cvg.game\_performance.voice\_actor\patharrow Arleen Sorkin}\\[6pt]

\texttt{Harley Quinn\patharrow film.film\_character.portrayed\_in\_films\patharrow m.0j6pcwz\patharrow \\ film.performance.actor\patharrow Chrissy Kiehl}\\[6pt]

\texttt{Harley Quinn\patharrow tv.tv\_character.appeared\_in\_tv\_program\patharrow m.02wm18b\patharrow \\ tv.regular\_tv\_appearance.actor\patharrow Mia Sara}\\[6pt]

\texttt{Harley Quinn\patharrow tv.tv\_character.appeared\_in\_tv\_program\patharrow m.0wz39vs\patharrow \\ tv.regular\_tv\_appearance.actor\patharrow Arleen Sorkin}\\[6pt]

\par}

\rule{\linewidth}{0.4pt}

\textbf{Ground-truth:}\\
Mia Sara, Hynden Walch, Arleen Sorkin

\end{tcolorbox}
\caption{Example on the retrieved reasoning paths by \proj}
\label{fig:example3}
\end{figure}

\begin{figure}[t]
\begin{tcolorbox}[colback=gray!5!white,
                  colframe=gray!75!black,
                  title=WEBQSP-WebQTest-1599]
\scriptsize

\textbf{Question:}\\
Where did Kim Jong-il die?

\rule{\linewidth}{0.4pt}

\textbf{Retrieved Paths:}\\[2pt]
{\raggedright
\texttt{Kim Jong-il\patharrow people.deceased\_person.place\_of\_burial\patharrow Kumsusan Palace of the Sun}\\[6pt]

\texttt{\textbf{Kim Jong-il\patharrow people.deceased\_person.place\_of\_death\patharrow Pyongyang}}\\[6pt]

\texttt{Kim Jong-il\patharrow people.deceased\_person.place\_of\_death\patharrow Pyongyang\patharrow \\ periodicals.newspaper\_circulation\_area.newspapers\patharrow Rodong Sinmun}\\[6pt]

\texttt{Kim Jong-il\patharrow people.deceased\_person.place\_of\_death\patharrow Pyongyang\patharrow \\ location.location.contains\patharrow Munsu Water Park}\\[6pt]

\texttt{Kim Jong-il\patharrow people.deceased\_person.place\_of\_death\patharrow Pyongyang\patharrow \\ location.location.contains\patharrow Pyongyang University of Science and Technology}\\[6pt]

\par}

\rule{\linewidth}{0.4pt}

\textbf{Ground-truth:}\\
Pyongyang

\end{tcolorbox}
\caption{Example on the retrieved reasoning paths by \proj}
\label{fig:example4}
\end{figure}

\begin{figure}[t]
\begin{tcolorbox}[colback=gray!5!white,
                  colframe=gray!75!black,
                  title=WEBQSP-WebQTest-1993]
\scriptsize

\textbf{Question:}\\
What language do they speak in Argentina?

\rule{\linewidth}{0.4pt}

\textbf{Retrieved Paths:}\\[2pt]
{\raggedright
\texttt{Argentina\patharrow location.country.languages\_spoken\patharrow Spanish Language}\\[6pt]

\texttt{Argentina\patharrow location.country.languages\_spoken\patharrow Yiddish Language}\\[6pt]

\texttt{Argentina\patharrow location.country.languages\_spoken\patharrow Guaraní language}\\[6pt]

\texttt{Argentina\patharrow location.country.languages\_spoken\patharrow Quechuan languages}\\[6pt]

\texttt{Argentina\patharrow location.country.languages\_spoken\patharrow Italian Language}\\[6pt]

\par}

\rule{\linewidth}{0.4pt}

\textbf{Ground-truth:}\\
Yiddish Language, Spanish Language, Quechuan languages, Italian Language, Guaraní language

\end{tcolorbox}
\caption{Example on the retrieved reasoning paths by \proj}
\label{fig:example5}
\end{figure}

\begin{figure}[t]
\begin{tcolorbox}[colback=gray!5!white,
                  colframe=gray!75!black,
                  title=CWQ-WebQTrn-962\_f0c57985929ee8b823983f6e5f104971]
\scriptsize

\textbf{Question:}\\
What actor played a kid in the film with a character named Veteran at War Rally?

\rule{\linewidth}{0.4pt}

\textbf{Retrieved Paths:}\\[2pt]
{\raggedright
\texttt{Forrest Gump\patharrow film.film\_character.portrayed\_in\_films\patharrow m.0jycvw\patharrow \\ film.performance.actor\patharrow Tom Hanks}\\[6pt]

\texttt{\textbf{Forrest Gump\patharrow film.film\_character.portrayed\_in\_films\patharrow m.02xgww5\patharrow  \\ film.performance.actor\patharrow Michael Connor Humphreys}}\\[6pt]

\texttt{Forrest Gump\patharrow common.topic.notable\_for\patharrow g.1258qx91g}\\[6pt]

\texttt{Veteran at War Rally\patharrow common.topic.notable\_for\patharrow g.12z7tmqks\patharrow \\ film.performance.character\patharrow Veteran at War Rally}\\[6pt]

\texttt{Veteran at War Rally\patharrow film.film\_character.portrayed\_in\_films\patharrow m.0y55311\patharrow \\ film.performance.actor\patharrow Jay Ross}\\[6pt]

\texttt{Veteran at War Rally\patharrow film.film\_character.portrayed\_in\_films\patharrow m.0y55311\patharrow \\ film.performance.character\patharrow Veteran at War Rally}\\[6pt]

\texttt{Forrest Gump\patharrow common.topic.notable\_for\patharrow g.1258qx91g\patharrow \\ book.book\_character.appears\_in\_book\patharrow Forrest Gump}\\[6pt]

\par}

\rule{\linewidth}{0.4pt}

\textbf{Ground-truth:}\\
Michael Connor Humphreys

\end{tcolorbox}
\caption{Example on the retrieved reasoning paths by \proj}
\label{fig:example6}
\end{figure}

\begin{figure}[t]
\begin{tcolorbox}[colback=gray!5!white,
                  colframe=gray!75!black,
                  title=CWQ-WebQTrn-2069\_9a4491f5f6a880a03bd96b8180bace4c]
\scriptsize

\textbf{Question:}\\
If I were to visit the governmental jurisdiction where Ricardo Lagos holds an office, what languages do I need to learn to speak?

\rule{\linewidth}{0.4pt}

\textbf{Retrieved Paths:}\\[2pt]
{\raggedright
\texttt{\textbf{Ricardo Lagos\patharrow government.politician.government\_positions\_held\patharrow m.0nbbvk0\patharrow \\ government.government\_position\_held.office\_position\_or\_title\patharrow President of Chile}}\\[6pt]

\texttt{\textbf{Ricardo Lagos\patharrow people.person.nationality\patharrow Chile\patharrow \\ location.country.official\_language\patharrow Spanish Language}}\\[6pt]

\texttt{Ricardo Lagos\patharrow people.person.nationality\patharrow Chile\patharrow \\ base.mystery.cryptid\_area\_of\_occurrence.cryptid\_s\_found\_here\patharrow Giglioli's Whale}\\[6pt]

\texttt{Ricardo Lagos\patharrow people.person.place\_of\_birth\patharrow Santiago\patharrow \\ location.location.contains\patharrow Torre Santa Maria}\\[6pt]

\texttt{Ricardo Lagos\patharrow people.person.nationality\patharrow Chile\patharrow \\  location.country.second\_level\_divisions\patharrow Choapa Province}\\[6pt]

\texttt{Ricardo Lagos\patharrow people.person.place\_of\_birth\patharrow Santiago\patharrow \\ location.location.time\_zones\patharrow Chile Time Zone}\\[6pt]

\texttt{Ricardo Lagos\patharrow people.person.place\_of\_birth\patharrow Santiago\patharrow \\ location.location.contains\patharrow Santiago Metropolitan Park}\\[6pt]

\par}

\rule{\linewidth}{0.4pt}

\textbf{Ground-truth:}\\
Spanish Language, Mapudungun Language, Aymara language, Rapa Nui Language, Puquina Language

\end{tcolorbox}
\caption{Example on the retrieved reasoning paths by \proj}
\label{fig:example7}
\end{figure}

\begin{figure}[t]
\begin{tcolorbox}[colback=gray!5!white,
                  colframe=gray!75!black,
                  title=CWQ-WebQTrn-105\_64489ea2de4b116070d33a0ebcfd4866]
\scriptsize

\textbf{Question:}\\
What currency is used in the country in which Atef Sedki held office in 2013?

\rule{\linewidth}{0.4pt}

\textbf{Retrieved Paths:}\\[2pt]
{\raggedright
\texttt{\textbf{Atef Sedki\patharrow government.politician.government\_positions\_held\patharrow m.0g9442f\patharrow \\ government.government\_position\_held.jurisdiction\_of\_office\patharrow Egypt}}\\[6pt]

\texttt{Atef Sedki\patharrow people.person.place\_of\_birth\patharrow Tanta\patharrow location.location.people\_born\_here\patharrow Atef Ebeid}\\[6pt]

\texttt{Atef Sedki\patharrow people.person.place\_of\_birth\patharrow Tanta\patharrow \\ location.statistical\_region.population\patharrow g.1jmcbhfn7}\\[6pt]

\texttt{Atef Sedki\patharrow people.person.place\_of\_birth\patharrow Tanta\patharrow \\ location.location.people\_born\_here\patharrow Ramadan Abdel Rehim Mansour}\\[6pt]

\texttt{Atef Sedki\patharrow people.person.nationality\patharrow Egypt\patharrow \\ organization.organization\_scope.organizations\_with\_this\_scope\patharrow Reform and Development Misruna Party}\\[6pt]

\texttt{Atef Sedki\patharrow people.person.nationality\patharrow Egypt\patharrow \\ organization.organization\_scope.organizations\_with\_this\_scope\patharrow Islamist Bloc}\\[6pt]

\texttt{Atef Sedki\patharrow people.person.nationality\patharrow Egypt\patharrow \\ location.location.events\patharrow Crusader invasions of Egypt}\\[6pt]

\texttt{Atef Sedki\patharrow government.politician.government\_positions\_held\patharrow m.0g9442f\patharrow \\ government.government\_position\_held.jurisdiction\_of\_office\patharrow Egypt\patharrow  \\ organization.organization\_scope.organizations\_with\_this\_scope\patharrow Free Egyptians Party}\\[6pt]

\par}

\rule{\linewidth}{0.4pt}

\textbf{Ground-truth:}\\
Egyptian pound

\end{tcolorbox}
\caption{Example on the retrieved reasoning paths by \proj}
\label{fig:example8}
\end{figure}

\begin{figure}[t]
\begin{tcolorbox}[colback=gray!5!white,
                  colframe=gray!75!black,
                  title=CWQ-WebQTest-213\_cbbd86314870b15371b43439eb40587a]
\scriptsize

\textbf{Question:}\\
What celebrities did Scarlett Johansson have romantic relationships with that ended before 2006?

\rule{\linewidth}{0.4pt}

\textbf{Retrieved Paths:}\\[2pt]
{\raggedright
\texttt{Scarlett Johansson\patharrow base.popstra.celebrity.dated\patharrow m.065q6ym\patharrow \\ base.popstra.dated.participant\patharrow Ryan Reynolds}\\[6pt]

\texttt{Scarlett Johansson\patharrow base.popstra.celebrity.dated\patharrow m.065q9sh\patharrow \\ base.popstra.dated.participant\patharrow Josh Hartnett}\\[6pt]

\texttt{Scarlett Johansson\patharrow base.popstra.celebrity.dated\patharrow m.064jrnt\patharrow \\ base.popstra.dated.participant\patharrow Josh Hartnett}\\[6pt]

\texttt{\textbf{Scarlett Johansson\patharrow base.popstra.celebrity.dated\patharrow m.064tt90\patharrow \\ base.popstra.dated.participant\patharrow Justin Timberlake}}\\[6pt]

\texttt{Scarlett Johansson\patharrow base.popstra.celebrity.breakup\patharrow m.064ttdz\patharrow \\ base.popstra.breakup.participant\patharrow Justin Timberlake}\\[6pt]

\texttt{\textbf{Scarlett Johansson\patharrow base.popstra.celebrity.dated\patharrow m.065q1sp\patharrow \\ base.popstra.dated.participant\patharrow Jared Leto}}\\[6pt]

\texttt{Scarlett Johansson\patharrow base.popstra.celebrity.dated\patharrow m.065ppwb\patharrow \\ base.popstra.dated.participant\patharrow Patrick Wilson}\\[6pt]

\texttt{Scarlett Johansson\patharrow base.popstra.celebrity.dated\patharrow m.065pwcr\patharrow \\ base.popstra.dated.participant\patharrow Topher Grace}\\[6pt]

\texttt{Scarlett Johansson\patharrow base.popstra.celebrity.breakup\patharrow m.064fpc6\patharrow \\ base.popstra.breakup.participant\patharrow nm1157013}\\[6pt]

\texttt{Scarlett Johansson\patharrow base.popstra.celebrity.dated\patharrow m.064fp5j\patharrow \\ base.popstra.dated.participant\patharrow nm1157013}\\[6pt]

\texttt{Scarlett Johansson\patharrow people.person.spouse\_s\patharrow m.0ygrd3d\patharrow \\ people.marriage.spouse\patharrow Ryan Reynolds}\\[6pt]

\par}

\rule{\linewidth}{0.4pt}

\textbf{Ground-truth:}\\
Justin Timberlake, Jared Leto

\end{tcolorbox}
\caption{Example on the retrieved reasoning paths by \proj}
\label{fig:example9}
\end{figure}

\begin{figure}[t]
\begin{tcolorbox}[colback=gray!5!white,
                  colframe=gray!75!black,
                  title=CWQ-WebQTrn-2569\_712922724a260d96fea082856cd21d6b]
\scriptsize

\textbf{Question:}\\
What sports facility is home to both the Houston Astros and Houston Hotshots?

\rule{\linewidth}{0.4pt}

\textbf{Retrieved Paths:}\\[2pt]
{\raggedright
\texttt{Houston Rockets\patharrow sports.sports\_team.arena\_stadium\patharrow Toyota Center}\\[6pt]

\texttt{Houston Hotshots\patharrow sports.sports\_team.arena\_stadium\patharrow NRG Arena}\\[6pt]

\texttt{Houston Rockets\patharrow sports.sports\_team.venue\patharrow m.0wz1znd\patharrow \\ sports.team\_venue\_relationship.venue\patharrow Toyota Center}\\[6pt]

\texttt{Houston Hotshots\patharrow sports.sports\_team.venue\patharrow m.0x2dzn8\patharrow \\ sports.team\_venue\_relationship.venue\patharrow Lakewood Church Central Campus}\\[6pt]

\texttt{\textbf{Houston Rockets\patharrow sports.sports\_team.venue\patharrow m.0wz8qf2\patharrow \\ sports.team\_venue\_relationship.venue\patharrow Lakewood Church Central Campus}}\\[6pt]

\texttt{\textbf{Houston Rockets\patharrow sports.sports\_team.arena\_stadium\patharrow Lakewood Church Central Campus}}\\[6pt]

\texttt{Houston Rockets\patharrow sports.sports\_team.location\patharrow Houston\patharrow \\ travel.travel\_destination.tourist\_attractions\patharrow Toyota Center}\\[6pt]

\par}

\rule{\linewidth}{0.4pt}

\textbf{Ground-truth:}\\
Lakewood Church Central Campus

\end{tcolorbox}
\caption{Example on the retrieved reasoning paths by \proj}
\label{fig:example10}
\end{figure}

\section{Prompt Template}
\label{prompt_template}
We provide the prompt template in this section, for rational paths identification, relation targeting, and hallucination detection, as shown in Figure~\ref{fig:get_rational_paths}-\ref{fig:hallucination_detect}.
\begin{figure}[!t]
\begin{tcolorbox}[colback=gray!5!white,
                  colframe=gray!75!black,
                  title=Prompt template for identifying rational paths]
\scriptsize

\tb{Example} \\ [5pt]
Given a question <\texttt{example question}>, the reasoning paths are: \\[5pt]

<\texttt{reasoning paths}> \\ [5pt]

The rational paths are: \\ [3pt]

<\texttt{Rational Paths}> \\ [5pt]

\tb{Explanation} \\ [3pt]

<\texttt{Explanation}> \\ [5pt]

\tb{Task} \\ [5pt]

Now given question <\texttt{question}>, the reasoning paths are: \\ [5pt]

<\texttt{Candidate Paths}> \\ [5pt]

Identify all the rational paths, and list below, with explanations: \\ [5pt]

<\texttt{Rational Paths}> \\ [3pt]

<\texttt{Explanations}> \\ [3pt]

\rule{\linewidth}{0.4pt}
\end{tcolorbox}
\caption{Prompt template for retrieving rational reasoning paths.}
\label{fig:get_rational_paths}
\end{figure}

\begin{figure}[!t]
\begin{tcolorbox}[colback=gray!5!white,
                  colframe=gray!75!black,
                  title=Prompt template for relation targeting]
\scriptsize

\tb{Example} \\ [5pt]
Given a question <\texttt{example question}>, the question entity is:<\texttt{question entity}>, the candidate relations for <\texttt{question entity}> are: <\texttt{relations}>. The possible relations for this question are: \\[5pt]

<\texttt{relations}>  \\[5pt]

\tb{Task} \\ [5pt]

Now given question <\texttt{question}>, question entity: <\texttt{question entity}>, relations with the entity: <\texttt{relations}>. List the possible relations for this question. \\ [5pt]

<\texttt{relations}> \\ [5pt]
\end{tcolorbox}
\caption{Prompt template for potential relation targeting.}
\label{fig:get_relation_cand}
\end{figure}

\begin{figure}[!t]
\begin{tcolorbox}[colback=gray!5!white,
                  colframe=gray!75!black,
                  title=Prompt template for hallucination detection]
\scriptsize

\tb{Task} \\ [5pt]
Given the question <\texttt{question}>, the retrieved reasoning paths are: \\[5pt]

<\texttt{reasoning paths}>  \\[5pt]

Now answer this question, and indicate whether you used the information provided above to answer it.

<\texttt{answers}> \\ [5pt]

\#\#\# LLM reasoner:\\ [2pt]
I have used the provided knowledge: <\texttt{Yes|No}>. \\ [3pt]

<\texttt{Explanations}>
\end{tcolorbox}
\caption{Prompt template for hallucination detection.}
\label{fig:hallucination_detect}
\end{figure}

\section{Limitations of \proj}
While \proj{} demonstrates strong performance in KGQA tasks and effectively improves generalization via structured graph retrieval, several limitations remain. 
First, our current approach focuses on enhancing retrieval through structural and relational inductive bias, but does not leverage the complementary strengths of LLM-based retrievers. Designing hybrid retrievers that combine the efficiency of graph-based reasoning with the flexibility and expressiveness of LLMs remains an open challenge. Such integration could potentially yield more robust and scalable KGQA systems.
Second, \proj{} assumes the availability of well-formed KGs and does not address errors or missing entities in the graph itself. Our future work may address the issue of the these aspects.

\section{Software and Hardware}
We conduct all experiments using PyTorch~\cite{paszke2019pytorchimperativestylehighperformance} (v2.1.2) and PyTorch Geometric~\cite{fey2019fast} on Linux servers equipped with NVIDIA A100 GPUs (80GB) and CUDA 12.1.

\end{document}